\documentclass[twoside,11pt]{article}

\usepackage{jmlr2e}
\usepackage{microtype}
\usepackage{graphicx}
\usepackage{caption}

\usepackage{tcolorbox}
\usepackage[framemethod=TikZ]{mdframed}
\mdfdefinestyle{MyFrame}{%
    linecolor=black,
    outerlinewidth=0.5pt,
    roundcorner=0pt,
    innertopmargin=\baselineskip,
    innerbottommargin=\baselineskip,
    innerrightmargin=5pt,
    innerleftmargin=15pt,
    backgroundcolor=gray!5!white}
\usepackage{subfig}
\usepackage{pifont}

\usepackage{bm}
\usepackage[ruled,linesnumbered]{algorithm2e}
\usepackage{booktabs} 
\usepackage{times}
\usepackage{bigints}
\usepackage{footnote}
\usepackage{amsmath}
\usepackage{bigints}
\usepackage{dsfont}
\usepackage{mathtools}
\usepackage{booktabs}
\DeclareMathOperator*{\argmax}{argmax}

\usepackage{thmtools, thm-restate}
\usepackage{xcolor}

\usepackage{hyperref}

\usepackage{todonotes}

\usepackage{lastpage}
\jmlrheading{22}{2021}{1-\pageref{LastPage}}{2/21;Revised
9/21}{9/21}{21-0120}{Henry B. Moss, David S. Leslie, Javier Gonz\'alez and Paul Rayson}

\ShortHeadings{GIBBON}{Moss, Leslie, Gonz\'alez and Rayson}
\firstpageno{1}

\begin{document}

\title{GIBBON: General-purpose Information-Based Bayesian OptimisatioN}

\author{\name Henry B. Moss \thanks{Work completed while at the STOR-i Centre for Doctoral Training, Lancaster University, UK.}\email henry.moss@secondmind.ai  \\
       \addr Secondmind.ai  \\ 
    Cambridge, UK
       \AND
        \name David S. Leslie \email d.leslie@lancaster.ac.uk \\
       \addr
       Lancaster University\\
    Lancaster, UK
     \AND
       \name Javier Gonz\'alez \email Gonzalez.Javier@microsoft.com \\
       \addr Microsoft Research\\
       Cambridge, UK
    \AND
    \name Paul Rayson \email p.rayson@lancaster.ac.uk \\
    \addr 
    Lancaster University\\
    Lancaster, UK}

\editor{Marc Peter Deisenroth}
\maketitle

\begin{abstract}
This paper describes a general-purpose extension of max-value entropy search, a popular approach for Bayesian Optimisation (BO). A novel approximation is proposed for the information gain --- an information-theoretic quantity central to solving a range of BO problems, including noisy, multi-fidelity and batch optimisations across both continuous and highly-structured discrete spaces. Previously, these problems have been tackled separately within information-theoretic BO, each requiring a different sophisticated approximation scheme, except for batch BO, for which no computationally-lightweight information-theoretic approach has previously been proposed. GIBBON (General-purpose Information-Based Bayesian OptimisatioN) provides a single principled framework suitable for all the above, out-performing existing approaches whilst incurring substantially lower computational overheads. In addition, GIBBON does not require the problem's search space to be Euclidean and so is the first high-performance yet computationally light-weight acquisition function that supports batch BO over general highly structured input spaces like molecular search and gene design. Moreover, our principled derivation of GIBBON yields a natural interpretation of a popular batch BO heuristic based on determinantal point processes. Finally, we analyse GIBBON across a suite of synthetic benchmark tasks, a molecular search loop, and as part of a challenging batch multi-fidelity framework for problems with controllable experimental noise. 
\end{abstract}


\newpage
\section{Introduction}
\label{sec:intro}
A popular solution for the optimisation of high-cost black-box functions is Bayesian optimisation \citep[BO]{mockus1978application}. By sequentially deciding where to make each evaluation as the optimisation progresses, BO can direct resources into evaluating promising areas of the search space to provide efficient optimisation. BO frameworks consist of two key components - a surrogate model and an acquisition function. By fitting a probabilistic surrogate model, typically a Gaussian process \citep[GP]{rasmussen2004gaussian}, to the previously collected objective function evaluations, we are able to quantify our current belief about which areas of the search space maximize our objective function. An acquisition function then uses this belief to predict the utility of making a particular evaluation, producing large values at `reasonable' locations. BO automatically evaluates the location that maximises this acquisition function and repeats until a sufficiently high-performing solution is found. A popular application of BO is hyper-parameter tuning, with successful applications in computer vision \citep{bergstra2013making}, text-to-speech \citep{moss2020boffin} and reinforcement learning  \citep{chen2018bayesian}. Of particular note are the recent extensions of BO beyond Euclidean search spaces, for example when optimising synthetic genes \citep{gonzalez2015bayesian,tanaka2018bayesian,Moss2020} or performing molecular search \citep{gomez2018automatic,griffiths2017constrained, vakili2020scalable}.

Various heuristic strategies have been developed to form BO acquisition functions, including Expected Improvement \citep[EI]{jones1998efficient}, Knowledge Gradient \citep[KG]{frazier2008knowledge} and Upper-Confidence Bound \citep[UCB]{srinivas2009gaussian}. More recently, a particularly intuitive and empirically effective class of acquisition functions has arisen based on information theory.  Information-theoretic BO seeks to reduce uncertainty in the location of high-performing areas of the search space, as measured in terms of differential entropy. Such entropy-reduction arguments have motivated the three primary information-theoretic acquisition functions of Entropy Search \citep[ES]{hennig2012entropy}, Predictive Entropy Search \citep[PES]{hernandez2014predictive} and Max-value Entropy Search \citep[MES]{wang2017max}, differing in their chosen measure of global uncertainty and employed approximation methods. Of particular popularity are acquisition functions based on MES, which reduce uncertainty in the maximum value attained by the objective function, a one-dimensional quantity. In contrast, both ES and PES seek to reduce uncertainty in the location of the maximum, a quantity which, as well as being well-defined only for Euclidean search spaces, requires prohibitively expensive approximation schemes. Due to the large number of acquisition function evaluations required to identify the next query point for each BO step, computational complexity is an important practical consideration when designing acquisition functions, particularly for applications with structured search spaces containing combinatorial elements.

Although the advent of MES acquisition functions has enabled the application of information-theoretic BO beyond problems with low-dimensional Euclidean search spaces, MES can not yet be regarded as a general-purpose acquisition function for two reasons. \begin{enumerate}
    \item Firstly, the existing extensions of MES supporting common BO extensions like Multi-fidelity BO \citep{mumbo} and batch BO \citep{takeno2020multi} require additional approximations beyond those of vanilla MES, typically through the numerical integration of low-dimensional integrals. Multi-fidelity BO (also known as multi-task BO) leverages cheap approximations of the objective function to speed up optimisation, for example through exploiting coarse resolution simulations when calibrating large climate models \citep{priess2011surrogate}, whereas batch BO allows multiple objective function evaluations to be queried in parallel, a scenario arising often in science applications, for example when training a collection of robots to cook \citep{junge2020improving}. Therefore, although still cheaper than their ES- and PES-based counterparts, extensions of MES for multi-fidelity and batch BO do not inherit the simplicity and low-cost of vanilla MES. 
    \item {Secondly, missing from the current extensions of MES is a computationally efficient method for general batch BO.} Asynchronous batch  BO supports scenarios where each of $B$ workers are allocated individually to evaluate different areas of the search space, returning queries and being re-allocated one by one. 
    In contrast, synchronous batch BO considers scenarios where where $B$ workers are to be allocated in parallel, as is the case for many real-world settings including those relying on wet-lab evaluations, physical experiments, or any framework where workers do not have sufficient autonomy to be controlled separately. {\cite{takeno2020multi} propose an low-cost MES formulation suitable for asynchronous batch BO, however, their proposed extension of this method to synchronous batch BO (a distinction discussed in depth by \cite{kandasamy2018parallelised}) require prohibitively expensive approximations.} Consequently, synchronous batch applications of MES have so far relied on generic batch heuristics suitable for any BO acquisition function, including greedy allocation through local penalisation \citep{gonzalez2016batch,alvi2019asynchronous} or using probabilistic repulsion models like determinantal point processes \citep{kathuria2016batched,dodge2017open}, both of which support only Euclidean search spaces.
\end{enumerate}

In this work we provide a single generalisation of MES suitable for BO problems arising from any combination of noisy, batch, single-fidelity, and multi-fidelity optimisation tasks. Crucially, unlike existing extensions of MES, our general-purpose acquisition function retains the computational cost of vanilla MES, with no requirement for numerical integration schemes. Therefore, we provide the first high-performing yet computationally light-weight framework for synchronous batch BO suitable for search spaces consisting of discrete structures.  

Our primary contributions are as follows:
\begin{enumerate}
    \item We propose an approximation for a general-propose extension of MES named General-purpose Information-Based Bayesian Optimisation (GIBBON). This approximation enables application of MES to a wide variety of problems, including those with combinations of synchronous batch BO, multi-fidelity BO and non-Euclidean highly-structured input spaces.
    \item Analysis of GIBBON leads to a  novel connection between information-theoretic search, determinantal point processes \citep[DPP]{kulesza2012determinantal} and local penalisation \citep{gonzalez2016batch}, providing { currently missing} theoretical justification for key attributes of these two popular heuristics previously chosen arbitrarily by users.
    \item We analyse the computational complexity of GIBBON in the wider context of information-theoretic acquisition functions, providing {a} comprehensive evaluation of the computational overheads of information-theoretic BO.
    \item
    We demonstrate the performance of GIBBON across a suite of popular benchmark optimisation tasks, including the first application of information-theoretic acquisition functions to high-cost string optimisation tasks and a sophisticated batch multi-fidelity framework for BO under controllable observation noise.
\end{enumerate}

The remainder of the paper is structured as follows. Section \ref{sec:mes} reviews prior work on MES and introduces the extended acquisition function that will be the focus of this work. In section \ref{sec:approx}, we propose the GIBBON acquisition function, before examining  GIBBON in the context of existing heuristics for batch BO (Section \ref{sec:batch}). In Section \ref{sec:computationalcomplexity}  we consider the computational complexity of GIBBON in the wider context of information-theoretic BO. Finally, Section \ref{sec:experiments} provides a thorough empirical evaluation.

\section{Max-value Entropy Search for Black-Box Function Optimisation}
\label{sec:mes}

We now introduce max-value entropy search (MES) for BO, providing an information-theoretic motivation for the general-purpose framework that is the focus of this manuscript. We then introduce existing work that has applied more restrictive formulations of MES to deal with specific BO tasks, before briefly summarising additional popular acquisition functions that are not based on MES.

BO routines seek the global maximum
\begin{align*}
    \textbf{x}^*=\argmax_{\textbf{x}\subset\mathcal{X}} g(\textbf{x})
\end{align*}
of a `smooth' but expensive to evaluate black-box function $g:\mathcal{X}\rightarrow\mathds{R}$. By sequentially choosing where and how to make each evaluation, BO directs resources into promising areas to efficiently explore the search space $\mathcal{X}\subset\mathds{R}^d$ and provide fast optimisation. In its simplest formulation (henceforth referred to as standard BO), BO controls the locations $\textbf{x}\in\mathcal{X}$ at which to collect (potentially noisy) queries of the objective function. A more general framework is that of  multi-fidelity BO \citep{swersky2013multi} (also known as multi-task BO), where the `quality' of each function query can also be controlled, for example by choosing the level of noise or bias across a (possibly continuous) space of fidelities $\textbf{s}\in\mathcal{F}$. If these lower-fidelity estimates of $g$ are cheaper to evaluate, then BO has access to cheap but approximate information sources that can be used to efficiently maximise $g$. In practical terms, each step of multi-fidelity BO needs to choose a location-fidelity pair $\textbf{z}=(\textbf{x},\textbf{s})\in\mathcal{Z}=\mathcal{X}\times\mathcal{F}$ upon which to make the next evaluation. A further extension arises as batch BO, where we wish to exploit parallel resources by choosing a set of $B\geq 1$ locations $\{\textbf{z}_1,..,\textbf{z}_B\}\in\mathcal{Z}^B$ to be evaluated in parallel.

BO's decisions are governed by two primary components - a surrogate model and an acquisition function. The surrogate model makes probabilistic predictions of the objective function at not-yet-evaluated locations using the already collected location-evaluation tuples $D_n=\{(\textbf{z}_i,{y_{\textbf{z}_i}})\}_{i=1,..,n}$. The most most popular choice of model is  a Gaussian process \citep[GP]{rasmussen2004gaussian}. GPs provide non-parametric regression over all functions of a smoothness controlled by a kernel $k:\mathcal{X}\times\mathcal{X}\rightarrow\mathds{R}$. Crucially, our GP conditioned on $D_n$ produces a tractable Gaussian predictive distribution that quantifies our current belief about the objective function across the whole search space. GP models can also be defined for multi-fidelity optimisation tasks \citep{kennedy2000predicting,le2014recursive,klein2016fast,perdikaris2017nonlinear,cutajar2019deep} and when modelling highly-structured input spaces likes strings \citep{beck2017learning}, trees \citep{beck2015learning} and molecules \citep{moss2020gaussian}.

Given such a probabilistic model over the search space, all that remains to perform an iteration of BO is an acquisition function measuring the utility of making evaluations. The Max-value Entropy Search (MES) of \cite{wang2017max}, with similar formulations considered by \cite{hoffman2015output} and \cite{ru2018fast},  seeks to query the objective function at locations that reduce our current uncertainty in the maximum value of our objective function $g^*=\max_{\textbf{x}\in\mathcal{X}} g(\textbf{x})$. In information theory \citep[see][for a comprehensive introduction]{cover2012elements}, uncertainty in the unknown $g^*$ is measured by its differential entropy $H(g^*|D_n)=-\mathds{E}_{g^*}\left[\log p(g^*)\right]$, where $p$ is the predictive probability distribution function for $g^*$ (as induced by the surrogate model). In particular, the utility of making an evaluation is measured as the reduction in the uncertainty of $g^*$ it provides, a quantity known as the mutual information ($\textrm{MI}$). 

Although initially proposed for just standard BO problems, an MES-based search strategy can be readily formulated for the general batch multi-fidelity framework (described above) by measuring the utility of evaluating a batch of fidelity evaluations as their joint mutual information with the maximum value. { We henceforth refer to this general formulation, formally expressed in Definition \ref{def:GMES}, as General-purpose MES (GMES).}

\begin{definition}[The GMES acquisition function]
The GMES acquisition function is defined as
\begin{align}
 \alpha_n^{\textrm{GMES}}(\{\textbf{z}_i\}_{i=1}^B)\coloneqq& \textrm{MI}(g^*;\{y_{\textbf{z}_i}\}_{i=1}^B|D_n)\nonumber\\=&H(g^*|D_n)-\mathds{E}_{\{y_{\textbf{z}_i}\}_{i=1}^B}\left[H(g^*|D_n\cup\{y_{\textbf{z}_i}\}_{i=1}^B)\right],\label{GMES}\end{align} where $\{\textbf{z}_i\}_{i=1}^B$ denotes the location-fidelity pairs of the batch elements and $y_{\textbf{z}}$ denote the yet-unobserved results of querying location-fidelity pair $\textbf{z}=(\textbf{x},\textbf{s})\in\mathcal{X}\times\mathcal{F}$.
\label{def:GMES}
\end{definition}
Note that standard BO, batch BO and multi-fidelity BO are trivial special cases of this general-purpose framework obtained by either or both of fixing the fidelity space $\mathcal{F}$ to a singleton containing just the true objective function or setting $B=1$. 

To provide resource-efficient optimisation, we must balance how much we expect to learn about $g^*$ with the computational cost of the evaluations. Therefore, following the arguments of \cite{swersky2013multi}, each BO step chooses to evaluate the set of $B$ locations that maximises the cost-weighted mutual information, i.e
\begin{align*}
    \{\textbf{z}_{|D_n|+1},..,\textbf{z}_{|D_n|+B}\}=\argmax_{\{\textbf{z}_i\}_{i=1}^B\in\mathcal{Z}^B}\frac{\alpha^{\textrm{GMES}}_n(\{\textbf{z}_i\}_{i=1}^B)}{c(\{\textbf{z}_i\}_{i=1}^B)},
\end{align*} where $c:\mathcal{Z}^B\rightarrow\mathds{R}^+$ measures the costs of evaluating the batch. This cost function could be known \textit{a priori} or estimated from observed costs \citep{snoek2012practical}.
The optimisation of acquisition functions is known as the \textit{inner-loop} maximisation and, when considering continuous search spaces, is typically performed with a gradient-based optimiser. For discrete search spaces it is common to use local optimisation routines like DIRECT  \citep{jones1993lipschitzian} or genetic algorithms \citep{Moss2020}. For search spaces with discrete and continuous dimensions, hybrid optimisers can be used \citep{ru2019bayesian}.

Unfortunately, calculating GMES in its full generality is challenging and providing a practically viable approximation strategy is the major contribution of this work. The primary difficulty in its computation arises from the lack of closed-form expression for the distribution of $g^*$, as required for all differential entropy calculations. We now end this section by discussing the three scenarios where specific sub-cases of GMES have already been used to provide highly effective BO --- a noiseless variant of standard BO, multi-fidelity BO, and a special case of batch BO.

\subsection{Max-value Entropy Search for noiseless standard BO}

Firstly, we consider the original MES formulation of \cite{wang2017max}, where they perform standard BO with noiseless observations. This acquisition function is formally expressed as 
\begin{align}
    \alpha_n^{\textrm{MES}}(\textbf{x})\coloneqq \textrm{MI}(y_{\textbf{x}};g^*|D_n)=H(y_{\textbf{x}}|D_n)-\mathds{E}_{g^*}\left[H(y_{\textbf{x}}|g^*,D_n)|D_n\right].
    \label{eq:MES_orig}
\end{align} 
Here, the symmetric property of mutual information has been used to swap $y_\textbf{x}$ and $g^*$ in its definition, yielding an equivalent (albeit less intuitive) expansion. Crucially, the first term of the expansion of  (\ref{eq:MES_orig})  is now simply the entropy of a multivariate Gaussian distribution with a convenient closed-form. Moreover, \cite{wang2017max} note that under the assumption of exact objective function evaluations (where $y_{\textbf{x}}=g(\textbf{x})$),  the distribution of $y_{\textbf{x}}$ conditional on its maximum possible value (i.e knowing that $y_{\textbf{x}}\leq g^*$) is simply that of a truncated Gaussian, also with a closed-form differential entropy. All that remains to calculate MES is to approximate an expectation over $g^*$. \cite{wang2017max} build a Monte-Carlo estimate of the expectation with a set of samples $\mathcal{M}$ from $g^*$, providing a closed-form approximation of MES as
\begin{align}
    \alpha^{MES}_n(\textbf{x})\approx \frac{1}{|\mathcal{M}|}\sum_{m\in\mathcal{M}}\left[\frac{\gamma_{\textbf{x}}(m)\phi\left(\gamma_{\textbf{x}}(m)\right)}{2\Phi\left(\gamma_{\textbf{x}}(m)\right)}-\log\Phi\left(\gamma_{\textbf{x}}(m)\right)\right],
    \label{eq:MES}
\end{align}
where  $\Phi$ and $\phi$ are the standard normal cumulative distribution and probability density functions (as arising from the expression for the differential entropy of a truncated Gaussian) and $\gamma_{\textbf{x}}(m)=\frac{m-\mu_n(\textbf{x})}{\sigma_n(\textbf{x})}$. Here, $\mu_n(\textbf{x})$ and $\sigma^2_n(\textbf{x})$ are the predictive mean and standard deviation for the objective function value $g$ at location $\textbf{x}$ as easily extracted from our surrogate model. The set of sample max-values $\mathcal{M}$ is built by modelling the empirical cumulative distribution function of $g^*$ with a Gumbel distribution (see \cite{wang2017max} for details) which can be sampled to yield $M$ cheap but approximate sampled max-values. This Gumbel approximation provides a fast sampling strategy and has been successful across a wide range of BO applications  \citep{wang2017max,mumbo,moss2020bosh}
 
For the limited set of BO problems supported by this original MES acquisition function, MES has had great empirical success, typically outperforming other information-theoretic BO methods with an order of magnitude smaller computational overhead. However, once MES arguments are extended to support the more sophisticated BO frameworks (or even just to support noisy function evaluations), we will see that the second term of (\ref{eq:MES}) is no longer (the expectation of) the differential entropy of a truncated Gaussian and additional approximations have to be made.

\subsection{Max-value Entropy Search for multi-fidelity BO}

MES-based search strategies have also been previously used for multi-fidelity BO through the MUlti-task Max-value Bayesian Optimisation (MUMBO) acquisition function of \cite{mumbo} (proposed in parallel by \cite{takeno2020multi}) and, just like original MES, MUMBO has been shown to perform highly efficient BO. However, unlike when collecting exact observations of $g$, fidelity evaluations $y_{\textbf{z}}|g^*$ no longer follow a truncated Gaussian distribution and instead follow an extended skew Gaussian distribution (as shown by \cite{mumbo} and re-derived in Section \ref{sec:approx}) which has no closed-form expression for its differential entropy \citep{azzalini1985class}. Therefore, the MUMBO acquisition function does not inherit all the computational savings of standard MES,  requiring numerical integration. Note that by considering a single fidelity system, where low-fidelity evaluations are just noisy observations of the true objective function, a multi-fidelity formulation of MES also serves as an extended standard (single-fidelity) MES suitable for when evaluations are contaminated with observation noise.

\subsection{Max-value Entropy Search for Batch BO}

Motivated by the empirical success of MES-based acquisition functions, it is natural to wonder if they can be used for batch BO. However, of the two popular batch scenarios of asynchronous and synchronous batches commonly considered in the BO literature, only asynchronous batch BO is currently supported by a MES-based acquisition function \citep{takeno2020multi}.  The primary practical distinction is that, while synchronous batch acquisition functions must be able to measure the total reduction in entropy provided by the joint evaluation of $B$ locations, asynchronous batch BO has only to measure the relative reduction in entropy provided by making a single evaluation whilst taking into account the $B-1$ pending evaluations. Through clever algebraic manipulations, \cite{takeno2020multi} require only single-dimensional numerical integrations when calculating  the relative entropy reduction required for asynchronous batch BO. Unfortunately, as demonstrated in Section \ref{sec:approx}, complex interactions between each of the $B$ fidelity evaluations $y_{\textbf{z}_i}$ once conditioned on $g^*$ (as present in the second term of (\ref{GMES})) prevents the approximation strategies employed by \cite{takeno2020multi} (or those of \cite{wang2017max} or \cite{mumbo}) being extended to the synchronous batch setting. In particular, a naive extension of \cite{takeno2020multi}'s approach requires the prohibitively expensive numerical approximations of $B$-dimensional multivariate Gaussian cumulative density functions. In this work, we propose a novel approximation strategy for (\ref{GMES}) completely free from  numerical integrations, thus providing the first computationally light-weight information-theoretic acquisition function for synchronous batch BO. 

\subsection{Alternatives to Max-value Entropy Search}
As discussed in Section \ref{sec:intro}, MES is not the only information-theoretic BO acquisition function and is a descendent of ES and PES. However, the original ES  and PES, as-well as their extensions for batch BO \citep{shah2015parallel,hernandez2017parallel} and multi-fidelity BO \citep{swersky2013multi,zhang2017information}, seek to reduce the differential entropy of the $d$-dimensional maximiser $\textbf{x}^*$ (rather than the single dimensional $g^*$ targeted by MES). The calculation of  this entropy is challenging, requiring sophisticated and expensive approximation strategies (see Section \ref{sec:computationalcomplexity}). As well as being substantially more expensive than MES, the reliance of ES and PES on coarse approximations means they provide less effective optimisation \citep{wang2017max,mumbo,takeno2020multi}. Moreover, the approximation strategy employed by PES restricts its use to only Euclidean search spaces

Of course, attempts have been made to adapt other standard acquisition functions to multi-fidelity and batch BO, with examples including EI \citep{picheny2010noisy,chevalier2013fast,marmin2015differentiating}, UCB  \citep{contal2013parallel,kandasamy2016gaussian,kandasamy2017multi} and KG \citep{wu2016parallel,wu2018continuous}. However, extensions of EI and UCB, although computationally cheap and often enjoying strong theoretical guarantees, are typically lacking in performance and even though KG-based methods can provide highly effective optimisation, their large computational cost restricts them to problems with function query costs large enough to overshadow very significant overheads (as demonstrated in Section \ref{sec:experiments}). For batch BO, additional heuristic strategies have been developed that are compatible with any acquisition function, with the most popular and empirically successful being the Local Penalisation of \cite{gonzalez2016batch} and DPP-based approach of \cite{kathuria2016batched} (see Section \ref{sec:batch} for a thorough discussion). Alternative but less performant heuristics include approaches based on Stein methods \citep{gong2019quantile} and Thompson sampling \citep{kandasamy2018parallelised}.

\section{A Novel Approximation of General-purpose Max-value Entropy Search}
\label{sec:approx}
In this section, we present the key theoretical contribution of this work: a novel approximation of the GMES acquisition function proposed in Section \ref{sec:mes}. In particular, we formulate GMES in terms of a the Information Gain (IG) --- a measure of entropy reduction often used when pruning decision tree classifiers \citep{raileanu2004theoretical} and when selecting features for statistical models of textual data \citep{yang1997comparative}. The remainder of the section then details a novel approximation strategy for the information gain based on simple well-known information-theoretic inequalities, before demonstrating explicitly how this IG approximation can be used to approximate the GMES acquisition function.

\subsection{GMES as a Function of Information Gain}

Recall our proposed GMES acquisition function (\ref{GMES}), defined as the mutual information between a set of $B$ fidelity evaluations and the objective function's maximum value $g^*$. As in the derivation of the original MES acquisition function (\ref{eq:MES_orig}), the symmetric property of mutual information can be used to yield the expansion
\begin{align}
    \alpha_n^{\textrm{GMES}}(\{\textbf{z}_i\}_{i=1}^B)\coloneqq H(\{y_{\textbf{z}_i}\}_{i=1}^B|D_n)-\mathds{E}_{g^*}\left[H(\{y_{\textbf{z}_i}\}_{i=1}^B|D_n,g^*)|D_n\right]. \label{GMESdef2}
\end{align}

For ease of notation, we now define $A_i=y_{\textbf{z}_i}$ and $C_i=g(\textbf{x}_i)$ for each of the $B$ candidate location-fidelity tuples $\textbf{z}_i$, as well as the multivariate random variables $\textbf{A}=(A_1,..,A_B)$ and $\textbf{C}=(C_1,..,C_B)$. The information gain is then defined as the reduction in the entropy of $\textbf{A}$ provided by knowing the maximal value of $C^*=\max\textbf{C}$, i.e.
\begin{align}
IG_n\left(\textbf{A},m|D_n\right)\coloneqq H(\textbf{A}|D_n)-H(\textbf{A}|C^*<m,D_n),
\label{IG_general}
\end{align}
Comparing (\ref{GMESdef2}) and (\ref{IG_general}), it follows that the GMES acquisition function can be expressed in terms of IG as 
 \begin{align}
    \alpha_n^{\textrm{GMES}}(\{\textbf{z}_i\}_{i=1}^B) = \mathds{E}_{m\sim g^*}\left[IG_n\left(\textbf{A},m|D_n\right)\right]. \nonumber
\end{align}

We can now see that efficiently calculating (\ref{IG_general}) in general scenarios will allow principled max-value entropy search across a wide range of BO settings. This goal is therefore the focus of the remainder of this section.

\subsection{Required Predictive Quantities}

Before presenting our proposed approximation for IG, it is convenient to discuss the distributional forms induced by our surrogate GP model. 
All random variables are now assumed to be conditioned on the arbitrary information set $D_n$, which, alongside references to $n$, is henceforth dropped from our notation. 

Courtesy of our GP surrogate model, we have that 
\begin{align}
\textbf{A}\sim N(\bm{\mu}^A,\Sigma^A),\quad \textbf{C}\sim N(\bm{\mu}^C,\Sigma^C)\quad \textrm{and} \quad \textrm{Corr}(A_i,C_i)=\rho_i,\nonumber
\end{align} for predictive means $\bm{\mu}^C,\bm{\mu}^A\in\mathds{R}^B$, predictive covariances $\Sigma^C,\Sigma^A\in\mathds{R}^{B\times B}$ and a vector of pairwise predictive correlations $\bm{\rho}\in\mathds{R}^B$ (\citeauthor{rasmussen2004gaussian}, \citeyear{rasmussen2004gaussian}; see Appendix \ref{appendix:quantitites} for details on how these predictive quantities are easily extracted from a GP).

In addition to these well-known distributional forms, we can exploit the specific conditional structure of our GP surrogate model (which we describe below and summarise in Figure \ref{DAG}) to derive the conditional distribution of the random variable $\textbf{A}$ given that $C^*<m$. In particular, our planned BO applications ensure that each $A_j$ is conditionally independent of $\{C_i\}_{i\neq j}$ given $C_j$. This condition holds trivially for single-fidelity BO, where the difference between each $A_i$ and $C_i$ is just independent Gaussian noise. For multi-fidelity BO, this condition corresponds exactly to the \textit{multi-fidelity Markov property} that is a key assumption underlying multi-fidelity GP modelling \citep{kennedy2000predicting,le2014recursive,perdikaris2017nonlinear}. 
This is not a restrictive assumption, with \cite {o1998markov} showing that the multivariate Markov property holds for any GP surrogate model with a kernel that can be factorised into a product of kernels, one defined across the fidelity and one across the search space.

\begin{figure}[t]
\centering
\includegraphics[width=0.5\textwidth]{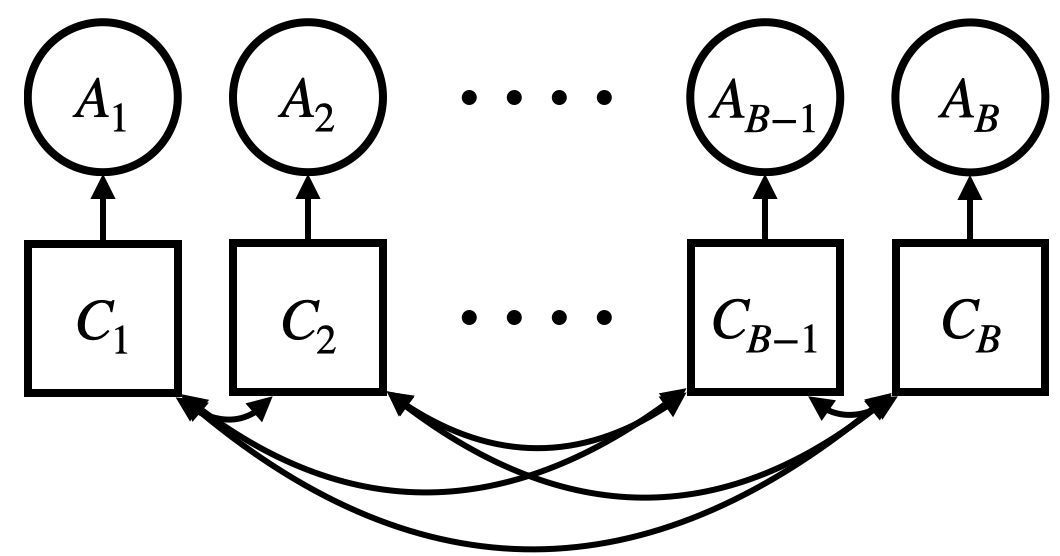}
\caption{The considered dependency structure between the two set of random variables $\{A_1,..,A_B\}$ and $\{C_1,..,C_B\}$. Arrows denote the direction of dependence and latent variables are drawn in squares.
\label{DAG}}
\end{figure}

Under these dependence assumptions, Theorem \ref{thm:distribution} provides the distribution of $\textbf{A}|C^*<m$ in closed-form, yielding  a probability density function that, to the authors' knowledge, has not been previously considered in the statistics literature. Theorem \ref{thm:distribution} provides our first intuition for why the efficient  calculation of the differential entropy $H(\textbf{A}|C^*<m)$ is challenging, i.e. the presence of the multivariate Gaussian cumulative density in its probability density function.

\begin{restatable}[Distribution of $\textbf{A}$ given $C^*<m$]{theorem}{dist}
\label{thm:distribution}
Consider two {$B$-dimensional} multivariate Gaussian random variables $\textbf{A}$ and $\textbf{C}$ where $\textbf{C}\sim N(\bm{\mu}^C,\Sigma^C)$ and each individual component of $\textbf{A}$ is distributed as $A_j\sim N(\mu^A_j,\Sigma^A_{j,j})$. Suppose further that each each pair $\{A_j,C_j\}$ are jointly Gaussian with correlation $\rho_j$, and that each $A_j$ is conditionally independent of $\{C_i\}_{i\neq j}$ given $C_j$. Define $C^*=\max \textbf{C}$. Then the conditional density of $\textbf{A}$ given that $C^*<m$ is given by
\begin{align*}
\frac{1}{\mathds{P}(C^*<m)}\phi_{{\textbf{X}_1}}(\textbf{a})\Phi_{{\textbf{X}_2}}(\textbf{m}),
\end{align*}
where $\textbf{m}=(m,..,m)\in\mathds{R}^B$ and $\phi_{{\textbf{X}_1}}$ and $\Phi_{{\textbf{X}_2}}$ are the probability density and cumulative density functions for the multivariate Gaussian random variables
\[{\textbf{X}_1}\sim \textbf{N}\left(\bm{\mu}^A,S +D\Sigma^C D\right)\quad\textrm{and}\quad{\textbf{X}_2}\sim \textbf{N}\left(\bm{\mu}^C+\Sigma^{-1}DS^{-1}(\textbf{a}-\bm{\mu}^A),\Sigma^{-1}\right),\]
where $\Sigma^A=D\Sigma^CD+S$ for $D$ and $S$, diagonal matrices with elements $D_{j,j}=\rho_j\sqrt{\frac{{\Sigma^A_{j,j}}}{\Sigma^C_{j,j}}}$ and $S_{j,j}=(1-\rho_j^2){\Sigma^A_{j,j}}$, and $\Sigma = \left(\left(\Sigma^C\right)^{-1}+DS^{-1}D\right)$.
\end{restatable}
\begin{proof}
See Appendix \ref{proofdist}
\end{proof}

Note that in the uni-variate case (i.e $B=1$ and $C^*=C_1$), Theorem \ref{thm:distribution} collapses to the settings already considered when calculating MES and MUMBO in Section \ref{sec:mes}. Firstly, under the strong restriction that  $C_1=A_1$ (arising from BO without observation noise), $A_1|C^*<m$ follows the well-known truncated Gaussian distribution, which can be seen directly from Theorem \ref{thm:distribution} by setting $\rho_j=1$, $\mu_{j,j}^C=\mu_j^A$ and $\Sigma_{j,j}^C=\Sigma_{j}^A$. This truncated Gaussian has a simple analytical expression for its differential entropy which is exploited by standard MES. Secondly, if  $C_j$ and $A_j$ are not perfectly correlated, we see that the density of Theorem \ref{thm:distribution} reduces to that of an Extended Skew Gaussian (ESG) distribution \citep{azzalini1985class} as required for the MUMBO acquisition function (see Appendix A of \cite{mumbo}). Although the differential entropy of an ESG has no closed-form expression \citep{arellano2013shannon}, we will later exploit the fact that its variance has an analytical form \begin{align}
    {\rm Var}(A_j|C_j<m)&=\Sigma^A_{j}\left(1-\rho^2_j\frac{\phi(\gamma_j(m))}{\Phi(\gamma_j(m))}\left[\gamma_j(m)+\frac{\phi(\gamma_j(m))}{\Phi(\gamma_j(m))}\right]\right),
    \label{moments}
\end{align} where $\gamma_j(m)=(m-\mu^C_j)/\sqrt{\Sigma^C_{j,j}}$. We stress that, due to the complex interactions between each $A_j|C^*<m$, the joint distribution of $\textbf{A}|C^*<m$ is not the multivariate ESG discussed by \cite{azzalini1996multivariate}).

\subsection{Approximating Information Gain}

We now present a lower bound $\textrm{IG}^{\textrm{APPROX}}$ for IG as Theorem \ref{thm:lowerbound}. This bound is to be used as an approximation $IG\approx IG^{\textrm{Approx}}$. We stress that replacing the maximisation of an intractable quantity with the maximisation of a lower bound is a well established strategy in the ML literature, for example in variational inference \citep{blei2017variational}.

\begin{restatable}[A lower bound for information gain]{theorem}{lb}
\label{thm:lowerbound}
Under the assumptions of Theorem \ref{thm:distribution}, it holds that $IG(\textbf{A},m)\geq IG^{\textrm{Approx}}(\textbf{A},m)$, where
\begin{align}
IG^{\textrm{Approx}}\left(\textbf{A},m\right)\coloneqq \frac{1}{2}\log|R^A|-\frac{1}{2}\sum_{i=1}^B\log\left(1-\rho^2_i\frac{\phi(\gamma_i(m))}{\Phi(\gamma_i(m))}\left[\gamma_i(m)+\frac{\phi(\gamma_i(m))}{\Phi(\gamma_i(m))}\right]\right) \label{eq:LB},
\end{align} where $R^A\in\mathds{R}^{B\times B}$ is the predictive correlation matrix of $\textbf{A}$ with entries $R^A_{i,j}=\Sigma_{i,j}^A/\sqrt{\Sigma^A_{i,i}\Sigma^A_{j,j}}$. 
\end{restatable}
\begin{proof}

Recall the definition of information gain $IG\left(\textbf{A},m\right)\coloneqq H(\textbf{A})-H(\textbf{A}|C^*<m)$. The first term of IG is simply the differential entropy of a multivariate Gaussian distribution and so can be written in closed-form as $H(\textbf{A}) = \frac{1}{2}\log\left[(2\pi e)^B\big|\Sigma_A\big|\right]$, where $\big|\Sigma_A\big|$ is the determinant of the $B\times B$ co-variance matrix of $\textbf{A}$. Unfortunately calculating the second term of IG is significantly more complicated, with a closed form expression only in the limited cases discussed above. 

We now build an analytical upper bound for $H\left(\textbf{A}|C^*<m\right)$ by exploiting three common information-theoretic inequalities. As derived in \cite{cover2012elements}, we know that,
\begin{align}
    H(\textbf{A})\leq\sum_{i=1}^BH(A_i),\quad
       H(A_i|C^*<m)\leq H(A_i{ | C_i<m}),\quad \textrm{and}\quad
   H(A_i)\leq\frac{1}{2}\log2\pi{\rm e}\textrm{Var}(A_i),\nonumber
\end{align} 
{ where the second inequality is due to $\{C^*<m\}$ being a stronger condition than (i.e. implying that) $\{C_i<m\}$}.

Applying the first two of these inequalities in sequence to $\textbf{A}|C^*<m$  yields the upper-bound \[H(\textbf{A}|C^*<m)\leq \sum_{i=1}^BH(A_i|C_i<m).\] Then, as we know that $A_j|C_j<m$ is an ESG (with a closed form expression for its variance), we can apply the third information-theoretic inequality to yield the analytical upper bound
 \begin{align}
H(\textbf{A}|C^*<m)\leq&\frac{1}{2}\sum_{i=1}^B\log(2\pi e{\rm Var}(A_i|C_i<m))\nonumber\\
=& \frac{1}{2}\sum_{i=1}^B\log2\pi e\Sigma^A_{i,i}\left(1-\rho^2_j\frac{\phi(\gamma_i(m))}{\Phi(\gamma_i(m))}\left[\gamma_i(m)+\frac{\phi(\gamma_i(m))}{\Phi(\gamma_i(m))}\right]\right).\nonumber
\end{align} Substituting this upper bound into (\ref{IG_general}), we have a lower bound for the information gain
\begin{align}
IG^{\textrm{Approx}}\left(\textbf{A},m\right)\coloneqq
\frac{1}{2}\log\big|\Sigma^A\big|\nonumber-& \frac{1}{2}\sum_{i=1}^B\log\Sigma^A_{i,i}\left(1-\rho^2_j\frac{\phi(\gamma_i(m))}{\Phi(\gamma_i(m))}\left[\gamma_i(m)+\frac{\phi(\gamma_i(m))}{\Phi(\gamma_i(m))}\right]\right)\nonumber\\= \frac{1}{2}\log\big|\Sigma^A\big|+& \frac{1}{2}\log\prod_{i=1}^b\left(\Sigma^A_{i,i}\right)^{-1} - \nonumber\\& \frac{1}{2}\sum_{i=1}^b\log\left(1-\rho^2_j\frac{\phi(\gamma_i(m))}{\Phi(\gamma_i(m))}\left[\gamma_i(m)+\frac{\phi(\gamma_i(m))}{\Phi(\gamma_i(m))}\right]\right) \nonumber,
\end{align} which after defining the predictive correlation matrix $R^A$ (with entries $R^A_{i,j}=\Sigma_{i,j}^A/\sqrt{\Sigma^A_{i,i}\Sigma^A_{j,j}}$) and noting that
\begin{align*}
    \frac{1}{2}\log\big|\Sigma^A\big|+ \frac{1}{2}\log\prod_{i=1}^b(\Sigma_{i,i}^A)^{-1} &= \frac{1}{2}\log\left|
    \begin{pmatrix} \frac{1}{\sqrt{\Sigma_{1,1}^A}} & & 0\\ & \ddots & \\0 & & \frac{1}{\sqrt{\Sigma_{b,b}^A}}\end{pmatrix}
    \Sigma^A
    \begin{pmatrix} \frac{1}{\sqrt{\Sigma_{1,1}^A}} & & 0\\ & \ddots & \\0 & & \frac{1}{\sqrt{\Sigma_{b,b}^A}}\end{pmatrix}
    \right| \\
    &= \frac{1}{2}\log \big|R^A\big|,
\end{align*} provides the claimed expression.
\end{proof}

\subsection{GIBBON: General-purpose Information-Based Bayesian OptimisatioN}
\label{GIBBON}
We end this section with explicitly demonstrating how $IG_{\textrm{Approx}}$ can be used to approximate the GMES acquisition function. Recall that GMES can be expressed in terms of IG as  
\begin{align}
    \alpha_n^{\textrm{GMES}}(\{\textbf{z}_i\}_{i=1}^B) = \mathds{E}_{m\sim g^*}\left[IG_n\left(\textbf{A},m|D_n\right)\right]. \nonumber
\end{align}

We have already provided an approximation for IG and so all that remains to approximate GMES is to deal with its outer expectation over $g^*$. Following the arguments of \cite{wang2017max}, we build a Monte-Carlo approximation of this expectation using a Gumbel-based sampler. Therefore, given a set of sampled max-values $\mathcal{M}=\{m_1,..,m_M\}$ of $g^*|D_n$ and access to the predictive distributions
\begin{align}
\{y_{\textbf{z}_i}\}_{i=1}^B|D_n\sim N(\bm{\mu}^y,\Sigma^y),\quad \{g(\textbf{x}_i)\}_{i=1}^B|D_n\sim N(\bm{\mu}^g,\Sigma^g)\quad \textrm{and} \quad \textrm{Corr}(y_{\textbf{z}_i},g(\textbf{x}_i)|D_n)=\rho_i,\nonumber
\end{align} we can approximate GMES with \begin{align}\alpha_n^{\textrm{GIBBON}}( \{\textbf{z}\}_{i=1}^B) = \frac{1}{|\mathcal{M}|}\sum_{m\in \mathcal{M}}IG^{APPROX}(\{y_{\textbf{z}_1},..,y_{\textbf{z}_b}\},m).\nonumber\end{align} This construction is henceforth referred to as the General Information-Based Bayesian OptimisatioN (GIBBON) acquisition function and is defined as the closed-form expression in Definition \ref{def:gibbon} and demonstrated within a BO loop as Algorithm \ref{alg:gibbon}. 

\noindent\fbox{\begin{minipage}{\textwidth}
\begin{definition}[The GIBBON acquisition function.]
The GIBBON acquisition function is defined as
\begin{align}\alpha_n^{\textrm{GIBBON}}( \{\textbf{z}\}_{i=1}^B) = \frac{1}{2}\log\big|R\big|-\frac{1}{2|\mathcal{M}|}\sum_{m\in \mathcal{M}} \sum_{i=1}^B\log\left(1-\rho^2_i\frac{\phi(\gamma_i(m))}{\Phi(\gamma_i(m))}\left[\gamma_i(m)+\frac{\phi(\gamma_i(m))}{\Phi(\gamma_i(m))}\right]\right),\nonumber\end{align} where $R$ is the  correlation matrix with elements $R_{i,j}=\Sigma^y_{i,j}/\sqrt{\Sigma^y_{i,i}\Sigma^y_{j,j}}$ and $\gamma_i(m)=\frac{m-\mu^g_i}{\sqrt{\Sigma^g_{i,i}}}$.
\label{def:gibbon}
\end{definition}
\end{minipage}}

\begin{algorithm}
\caption{GIBBON for general-purpose BO tasks.}
\label{alg:gibbon}
\KwIn{Resource budget $R$, Batch size $B$,  Gumbel sample size $N$}

Initialise $n\leftarrow 0$ and spent resource counter $r\leftarrow 0$

Propose initial design $I$

\While{$r\leq R$}{

Begin new iteration $n\leftarrow n + 1$

Fit GP model to collected evaluations $D_n$ 

Simulate $N$ samples from $g^*|D_n$ 

Compute $\alpha_n^{\textrm{GIBBON}}$ as given by Definition \ref{def:gibbon}

Find $B$ locations $\{\textbf{z}_i\}_{i=1}^{B}$ maximising $\frac{\alpha^{\textrm{GIBBON}}_n(\{\textbf{z}_i\}_{i=1}^{B})}{c(\{\textbf{z}_i\}_{i=1}^{B})}$

Evaluate new locations and collect evaluations  $D_{n+1}\leftarrow D_{n} \bigcup\{(\textbf{z}_i,{y_{\textbf{z}_i}})\}_{i=1}^B $

Update spent budget $r\leftarrow r+c(\{\textbf{z}_i\}_{i=1}^{B})$

}
\KwOut{Believed maximiser $\argmax_{\textbf{x}\in D_n} g(\textbf{x})$}
\end{algorithm}

At first glance, GIBBON's  analytical form looks complex. However, as GIBBON contains only simple  algebraic operations, it can be easily calculated in just a few lines of code, unlike existing ES-based and PES-based acquisition functions and all existing extensions of MES (as discussed in depth in Section \ref{sec:computationalcomplexity}).  An important practical consideration for GIBBON is that, for continuous search spaces, it has accessible gradients that can easily be derived from its analytical expression, allowing efficient inner-loop optimisation.

We end this section with a visual analysis of the accuracy of the GIBBON approximation. We consider a standard BO task with exact objective function evaluations (i.e not multi-fidelity or batch optimisation) as, in this setting, the MES acquisition function provides an exact calculation of the entropy reductions. In Figure \ref{fig::2d} we see that the approximation provided by GIBBON is very close to the ground truth provided by MES, with GIBBON and MES sharing modes and differing only in areas of the space that would never be selected by BO, i.e those locations with very low utility.

\begin{figure}
\subfloat[MES acquisition function surface (ground truth).]{\includegraphics[width= 0.48\textwidth]{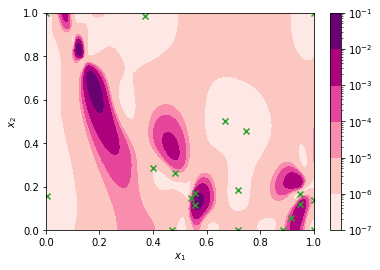}}
\subfloat[GIBBON acquisition function surface.]{\includegraphics[width= 0.48\textwidth]{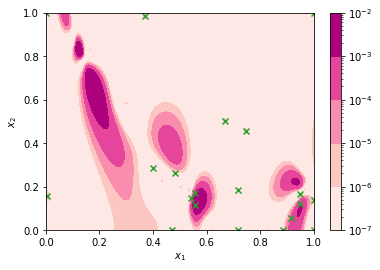}}
\caption{Comparison of the MES and GIBBON acquisition functions for a two-dimensional BO task where MES can calculate entropy reductions exactly. The crosses denote the locations already queried by the BO routine. GIBBON provides a very close approximation of MES that reliably captures all its modes.}
        \label{fig::2d}
\end{figure}

\section{Relationship Between GIBBON and Heuristics for Batch Bayesian Optimisation }
\label{sec:batch}
We now provide insights into the batch capabilities of our GIBBON acquisition function by drawing equivalences between GIBBON and two popular heuristics for batch BO --- determinantal point processes (Section \ref{subsec:DPPrelationship}) and local penalisation (Section \ref{subsec:LPrelationship}). 

Recall that performing an iteration of BO requires the identification of optimal candidate points across the search space, i.e the maximisation of our acquisition function. For GIBBON, this \textit{inner-loop} maximisation task corresponds to allocating a batch of $B$ locations as
\begin{align}
    \{\textbf{z}_{|D_n|+1},..,\textbf{z}_{|D_n|+B}\} = \argmax_{\textbf{z}\in \mathcal{Z}} \alpha_n^{\textrm{GIBBON}}(\{\textbf{z}_i\}_{i=1}^B).\nonumber
\end{align}

Before introducing the two batch BO heuristics, it is convenient to provide an alternative expression for the GIBBON acquisition function. From Definition \ref{def:gibbon}, we see that the GIBBON acquisition function for a candidate batch of $B$ location-fidelity tuples can be decomposed into a sum of $B$ GIBBON acquisition function evaluated separately for each tuple with an additional determinant term as
\begin{align}\alpha_n^{\textrm{GIBBON}}( \{\textbf{z}\}_{i=1}^B) =& \frac{1}{2}\log\big|R\big|+\sum_{i=1}^B\alpha_n^{\textrm{GIBBON}}(\textbf{z}_i), \label{decomp}
\end{align}
where R is the predictive correlation matrix of the batch. Note that the first term of this decomposition encourages diversity within the batch (achieving high values for points with low predictive correlation) whereas the second term ensures that evaluations are targeted in areas of the search space providing large amounts of information about $g^*$.

\subsection{Relationship with Determinantal Point Processes}
\label{subsec:DPPrelationship}
We can now interpret GIBBON in the context of a popular heuristic approach for batch design based on probabilistic models of repulsion known as Determinantal Point Processes (DPPs) \citep{kulesza2012determinantal}. This comparison provides { previously missing} theoretical justification for choices of key DPP attributes which previously had to be chosen arbitrarily by practitioners.

DPPs provide a probability distribution over sets of points, such that sets of high-quality points (as measured by a quality function $q:\mathcal{X}\rightarrow\mathds{R}$) with a diverse spread (as measured by a similarity kernel $s:\mathcal{X}\times\mathcal{X}\rightarrow\mathds{R}^+$) occur with high probability. More precisely, a particular set of points $\{\textbf{x}_i\}_{i=1}^B$ occurs with probability.
\begin{align}
    \mathds{P}(\{\textbf{x}_j\}_{j=1}^B)\propto|L(\{\textbf{x}_j\}_{j=1}^B)|,
    \label{DPP}
\end{align}
where $L(\{\textbf{x}_j\}_{j=1}^B)$ is a $b\times b$ matrix with elements $L_{i,j}=q(\textbf{x}_i)q(\textbf{x}_j)s(\textbf{x}_i,\textbf{x}_j)$. 

Generating diverse but high-quality collections of points is exactly what we seek when allocating batches in BO problems. Unfortunately, a lack of understanding of how to choose appropriate quality functions and similarity kernels \textit{a-priori} have previously limited the performance of DPP methods in BO, with existing applications requiring users to plug in arbitrary choices. The primary complication is that the relative scales of $q$ and $s$ trade-off the quality and diversity of batches, and so, for high-performance BO, these measures must be carefully chosen to complement (rather than dominate) each other. Consequently, the most common approach for using DPPs for BO is as part of \textit{pure exploration} strategies, where the quality function is ignored ($q(\textbf{x})=1$) and a DPP with a radial basis function kernel as its similarity measure is sampled to allocate a whole batch \citep{dodge2017open}, or to allocate the $B-1$ elements remaining after choosing an initial point through a standard sequential BO routine \citep{kathuria2016batched}. Related approaches have also been used for high-dimensional BO \citep{wang2017batched}, where DPPs are used to sample a subset of the available search space dimensions. Note that these existing applications of DPPs to batch BO are  limited in scope, supporting only  single-fidelity problems over Euclidean search spaces, i.e those over which a standard similarity kernel can easily be defined.

We now explicitly show that our GIBBON acquisition function is equivalent to a DPP with specific choices of quality functions and similarity kernels. First define the exponential of our GIBBON acquisition function (with $B=1$) as a quality function $q^{\textrm{G}}(\textbf{z})=\exp\left(\alpha^{\textrm{GIBBON}}(\textbf{z})\right)$ and the predictive correlation (as specified by our GP surrogate model) as a similarity kernel $s^{\textrm{G}}(\textbf{z}_i,\textbf{z}_j)=R_{i,j}$. Then, after defining $L^{\textrm{G}}(\{\textbf{z}_j\}_{j=1}^B)$ as the matrix with elements $L^{\textrm{G}}_{i,j}=q^{\textrm{G}}(\textbf{z}_i)q^{\textrm{G}}(\textbf{z}_j)s^{\textrm{G}}(\textbf{z}_i,\textbf{z}_j)$, simple algebraic manipulations allow the batch GIBBON acquisition function (\ref{decomp}) to be expressed as  
\begin{align*}
    \alpha_n^{\textrm{GIBBON}}(\{\textbf{z}_j\}_{j=1}^B)=\frac{1}{2}log|L^{\textrm{G}}|,
\end{align*}
i.e the maximisation of our acquisition function corresponds to allocating the batch with maximal $|L^G|$, known as the \textit{maximum a posteriori} (MAP) problem of DPPs. This is known to be $NP$-hard \citep{ko1995exact}. However, the submodularity of DPPs ensures reasonable performance of greedy approximate solutions \citep[as demonstrated by][]{gillenwater2012near}, explaining the observed effectiveness of a greedy batch-filling strategy when optimising our  GIBBON acquisition function (see Section \ref{sec:experiments}).

Recasting GIBBON as a DPP  provides the first theoretical motivation for using DPPs for batch BO, with the particular choices of  quality and similarity function arising from our information-theoretical derivation leading to significant improvements over existing DPP heuristics (Section \ref{sec:experiments}). Moreover, we have greatly increased the generality of DPP-based BO, providing { a} formulation that supports multi-fidelity and structured search spaces, or any other framework using a surrogate model where posterior correlation is easily accessible.

\subsection{Relationship with Local Penalisation}
\label{subsec:LPrelationship} 
Another class of popular heuristics for batch BO are those based on local penalisation (LP) \citep{gonzalez2016batch,alvi2019asynchronous}. Rather than explicitly balancing the diversity and quality of batches as two additive contributions, LP methods apply a multiplicative scaling to down-weight an acquisition function around locations already present in the batch, thus ensuring the selection of a diverse set of points. We now show that GIBBON can be interpreted as a penalisation strategy and consequently, {we can} make an explicit link between DPP- and LP-based BO routines. By recasting GIBBON as a local penalisation, we are able to derive a novel theoretically-justified penalisation function that outperforms existing LP methods.

For any choice of acquisition function $\alpha_n:\mathcal{X}\rightarrow\mathds{R}$ taking positive values, an LP strategy greedily chooses the $i^{th}$ element of the $n+1^{th}$ batch as
\begin{align}
    \textbf{x}_{n+1,i}=\argmax_{\textbf{x}\in\mathcal{X}}\alpha_n\left(\textbf{x}\right)\prod_{j=1}^{i-1}\psi(\textbf{x};\textbf{x}_{n+1,j}), \nonumber
\end{align}
where $\psi(\textbf{x},\textbf{x}'):\mathcal{X}\times\mathcal{X}\rightarrow[0,1]$ is a \textit{penalisation function}. By requiring that $\psi(\textbf{x},\textbf{x}')$ is a non-increasing function of $||\textbf{x}-\textbf{x}'||$, we ensure that penalisation is largest when considering $\textbf{x}$ close to elements already present in the batch. The most popular penalisation function is the soft penaliser of \cite{gonzalez2016batch}
\begin{align}
    \psi_{soft}(\textbf{x},\textbf{x}')=\frac{1}{2}\textrm{erfc}(-z) \quad \textrm{for} \quad  z = \frac{1}{\sqrt{\sigma^2_n(\textbf{x}')}}\left(L||\textbf{x}-\textbf{x}'||-g^*+\mu_n(\textbf{x}')\right)\nonumber,
\end{align} 
where $\textrm{erfc}$ is the complementary error function and $g^*$ is the current believed optimum. An important practical consideration of LP routines is that their performance is sensitive to predicting a Lipschitz constant $L$ (i.e $|g(\textbf{x})-g(\textbf{x}')|\leq L||\textbf{x}-\textbf{x}'||\quad\forall\textbf{x},\textbf{x}'\in\mathcal{X}$), for which point-estimates must be carefully extracted from previous function evaluations. Note that this Lipschitz constant can only be defined for Euclidean search spaces.

We now show that allocating batches by performing a greedy maximisation of  GIBBON  can be interpreted as an LP routine for specific choices of acquisition and penalisation functions. Define a re-scaled GIBBON acquisition function $\alpha_n^{scaled}(\textbf{x})=\left(e^{\alpha_n^{gibbon}(\textbf{x})}\right)^2$ and a penaliser $\psi_{corr}(\textbf{x};\{\textbf{x}_{j}\}_{j=1}^{i-1})=\big|R(\{\textbf{x}_{j}\}_{j=1}^{i-1}\cup\{\textbf{x}\})\big|$ as the determinant of the batch's predictive correlation. After routine algebraic manipulations, we can see that allocating the $i^{th}$ element of the $n+1^{th}$ batch according to a greedy maximisation of our GIBBON acquisition function is equivalently expressed as 
\begin{align}
     \textbf{x}_{n+1,i}=&\argmax_{\textbf{x}\in\mathcal{X}}\alpha_n^{\textrm{GIBBON}}
     \left(\{\textbf{x}\}\cup\{\textbf{x}_{n+1,j}\}_{j=1}^{i-1}\right)\nonumber
     \\=& \argmax_{\textbf{x}\in\mathcal{X}}\alpha^{scaled}_n(\textbf{x})\psi_{corr}(\textbf{x};\{\textbf{x}_{n+1,j}\}_{j=1}^{i-1}), \nonumber
\end{align}
i.e. the predictive correlation term in GIBBON can be interpreted as a form of local penalisation. However, unlike $\psi_{soft}$ and the {hard penaliser of \cite{alvi2019asynchronous}}, $\psi_{corr}$ does not require the estimation of $L$, instead just using the easily accessible predictive correlation of our GP. In fact the superior performance of our proposed approach over existing LP methods suggests that complicated penalisation functions are not needed at all.

\section{The Computational Complexity of Information-theoretic Bayesian Optimisation}
\label{sec:computationalcomplexity}
In this final section before our experimental results, we analyse the computational overhead incurred by GIBBON and compare with all other existing information-theoretic acquisition functions, many of which are included in our experimental results of Section \ref{sec:experiments}.  {We discuss} the complexity of the information-theoretic acquisition functions mentioned in Sections \ref{sec:intro} and \ref{sec:mes}:  Entropy Search \citep[ES]{hennig2012entropy}, Predictive Entropy Search \citep[PES]{hernandez2014predictive} and its extensions PPES \citep{hernandez2017parallel} and MF-PES \citep{zhang2017information}, Max-value Entropy Search \citep[MES]{wang2017max} and its extensions MUMBO \citep{mumbo} and MF-MES \citep{takeno2020multi}, as well as the Fast Information-Theoretic BO of \citet[FITBO]{ru2018fast}. Although MFMES was originally designed for asynchronous batch BO, \cite{takeno2020multi} do discuss (in their Appendix D.4) an alteration that allows the support for synchronous batch BO problems but with large computational cost. It is this variant of MFMES that we consider in this section and for our experimental results (Section \ref{sec:experiments}).

The computational complexity of BO routines is hard to measure exactly as we do not know  \textit{a-priori} how many evaluations are required to maximise the highly multi-modal acquisition function in each inner loop. However, there are two main contributors to the computational cost of information-theoretic BO that can be analysed: a one-off initialisation calculation required to `prepare' the acquisition functions for each separate BO step, and the costs of each acquisition function query required for the inner-loop maximisation. These two complexity contributions are presented in Table \ref{table:complexity}, alongside a summary of the type of extended BO problems supported by each acquisition function, i.e whether they permit noisy, multi-fidelity, batch observations or non-Euclidean search spaces. We now derive the stated complexity results for initialisation and acquisition function query costs.

\begin{table}[t]
\setlength{\tabcolsep}{2pt}
\begin{tabular}{@{}l|lllllll@{}}
\toprule
 Method & Noise? &  \begin{tabular}[c]{@{}l@{}}Multi-\\ Fidelity ?\end{tabular} & 
  Batch? & \begin{tabular}[c]{@{}l@{}}Non-\\ Euclidean ?\end{tabular} &
 \begin{tabular}[c]{@{}l@{}}Initialisation\\ costs\end{tabular} & \begin{tabular}[c]{@{}l@{}}Acquisition\\ query costs\end{tabular} \\ \midrule
ES & $\checkmark$ & $\checkmark$ & $\times$&$\checkmark$&$n^2e^{2d} + e^{3d}$ & $n^2e^{d}$ \\
PES & $\checkmark$ &  $\times$& $\times$&$\times$&$n^2e^{2d}+(n+d)^3e^d$ & $n^2 + (n+d)e^d$ & \\
PPES & $\checkmark$ &  $\times$& $\checkmark$&$\times$&$n^2e^{2d}+(n+d)^3e^d$ & $B^2n^2 + (B^3 + n+d)e^d$ & \\
MF-PES & $\checkmark$ & $\checkmark$ & $\times$& $\times$&$n^2e^{2d}+(n+d)^3e^d$ & $n^2 + (n+d)e^d$ \\
FITBO & $\times$ & $\times$ &$\times$ & $\times$&$1$& $n^2$\\
MES &$\times$ & $\times$ & $\times$&$\checkmark$&$n^2e^d$ & $n^2$ \\
MUMBO & $\checkmark$ & $\checkmark$ &$\times$ &$\checkmark$&$n^2e^d$ & $n^2$ \\
MF-MES & $\checkmark$ & $\checkmark$ & $\checkmark$ &$\checkmark$&$n^2e^d$ & $B^2n^2+B^3+B^2$  \\
GIBBON & $\checkmark$ & $\checkmark$ & $\checkmark$&$\checkmark$&$n^2e^d$ & $B^2n^2+B^3$ \\
\bottomrule
\end{tabular}
\caption{Computational complexity of existing entropy-based acquisition functions. $d$ denotes the dimensions of the search space, $n$ is the number of observations already collection, and $B$ denotes batch size. Complexity results are correct to highest order terms only and ignore constant factors. We also summarise the types of BO problems supported by these acquisition functions (columns 1-4). {For example, although standard MES's calculations strategy assumes exact, single-fidelity and purely sequential evaluations, MES does support non-Euclidean search space.}}
\label{table:complexity}
\end{table}

\subsection{Acquisition Function Initialisation Costs}

All BO routines incur a computational cost at the start of each individual BO step through the fitting of the surrogate model.  The primary contribution to the cost of fitting a GP surrogate model on $n$ data points is an $n\times n$ matrix inversion, i.e an $O(n^3)$ computation. Extracting a single predictive mean or co-variance from this GP then costs $O(n)$ and $O(n^2)$, respectively. As the overhead of fitting the GP is incurred across all BO routines, we leave out its contribution from our complexity analysis. We instead focus purely on the initialisation overheads specific to each information-theoretic acquisition function incurred when collecting sets of samples required for their approximation strategies. This set is reused for all acquisition function evaluations during a single inner-loop maximisation but re-sampled for each BO step.

All the samples required for information-theoretic acquisition functions can be separated into two distinct classes --- those approximating single-dimensional quantities and those approximating quantities with the same dimensions as the search space. To paint a clear picture of computational cost, we consider BO problems with a search space of fixed dimension $d$ and focus primarily on how the costs scale with respect to $d$, the batch size $B$, and the number of previously queried points $n$. Although all sample sizes are user-controllable, the efficiency of the resulting acquisition function depends sensitively on appropriately large sample sizes (as demonstrated for PES and MES by \cite{wang2017max}). Therefore, sample sizes used when approximating $d$-dimensional quantities must grow exponentially as $O(e^d)$ in order to preserve approximation accuracy. In contrast, the sample sizes required for effective approximations of single dimensional quantities can be chosen independently of $d$ and so are denoted as $O(1)$ in our complexity analysis.

As discussed in Section \ref{sec:mes}, MES-based acquisition functions (including GIBBON), uses a Gumbel sampler to access samples of the  maximum value $g^*$. This sampler evaluates our GP surrogate model's posterior (at $O(n^2)$ cost) across $O(e^d)$ points to form a discretisation of the $d$-dimensional search space. Each of the required $O(1)$ samples of $g^*$ (a single dimensional quantity) can then be extracted with $O(1)$ cost, yielding an overall complexity of  $O(n^2e^d)$. As shown in Table \ref{table:complexity},  GIBBON's initialisation costs are substantially lower than those of the acquisition functions based on PES and ES. Only FITBO has a lower initialisation cost, however it has not seen widespread use as it  supports only noiseless standard BO tasks and employs a complicated construction requiring linear approximations of non-central $\chi^2$ process (operations not supported by GP libraries). For the ES and PES-based acquisition functions, which require samples from the $d$-dimensional objective function maximiser $\textbf{x}^*$,  initialisation costs are substantial. 

In ES, each sample of $\textbf{x}^*$ is the maximum of a sample function drawn from the GP across an $O(e^d)$ discretisation of the search space. Simulating these function draws requires a one-off $O(e^{3d})$ computation for the Cholesky factor of the predictive co-variance matrix evaluated across the discretisation, as accessed with an $O(n^2)$ cost for each of its $O(e^{2d})$ elements.  Consequently, the initialisation of ES incurs a sizeable $O(n^2e^{2d}+e^{3d})$ complexity scaling. PES also requires samples of $\textbf{x}^*$ but instead maximises the sample draws from a finite feature approximation of the GP surrogate model \citep{rahimi2008random}, requiring just an $O(n^2)$ cost for each of the required $O(e^d)$ samples. However, unlike ES, PES incurs the additional cost of pre-computing an $n+d$-dimensional matrix inversion for each sample. Therefore, PES has a total initialisation cost of $O(n^2e^d +(n+d)^3e^d))$. Note that the finite feature approximation employed by PES and its variants is only rigorously defined for GPs with stationary kernels and Euclidean search spaces.

\subsection{Acquisition Function Query Costs}
We now discuss the computational complexity of each individual acquisition function query. As highlighted in Table \ref{table:complexity}, not only does the GIBBON acquisition function match the lowest query costs attained by any information-theoretic acquisition functions, but it is suitable for standard, stochastic, multi-fidelity and batch optimisation.

To calculate GIBBON and the other MES-based acquisition functions, we require the joint predictive distribution across $B$ proposed batch locations. Accessing these $B^2$ predictive co-variance terms from a GP surrogate model and then taking its determinant cost $O(B^2n^2)$ and $O(B^3)$, respectively. Finally, GIBBON calculates an analytical expression for each of the $O(1)$ samples from $g^*$ and across each of the batch elements, yielding an overall complexity of $O(B^2n^2+B^3)$. MF-MES has a similar construction to GIBBON, but requires the additional calculation of a $B$-dimensional integral, for which a naive numerical approximation would require an $O(e^B)$ cost. {Following \cite{takeno2020multi}, this integral can also be evaluated using a sophisticated recursive strategy for calculating multi-variate Gaussian cumulative density functions with $O(B^2)$ cost, however, we found this routine to incur a large constant overhead that dominated our acquisition function calculations.
Similarly, we stress that although all MES-based acquisition functions have $O(n^2)$ cost (in the non-batch setting), FITBO, MUMBO and MF-MES all require additional numerical integrations (over GIBBON) that incur a significant constant cost factor that does not show in our highest order complexity analysis. }  Consequently, the experiments of Section \ref{sec:experiments} show that GIBBON is substantially cheaper than MUMBO and MF-MES in practice.

The ES and PES-based acquisition functions incur a substantially larger query cost than GIBBON. Their primary computational bottleneck is the requirement of separate calculations for each of their $O(e^d)$ samples of $\textbf{x}^*$. In ES, each evaluation requires an $n^2$ prediction from the GP for each location across a small $O(1)$-sized collection of points for each sampled $\textbf{x}^*$. In contrast, PES requires only a single prediction from the GP but additional $O(n+d)$ manipulations for each of its  $O(e^d)$ pre-computed kernel matrices. For batch BO, PPES requires $B^2$ GP predictions and a $B^3$ calculation to access the determinant of the batch's posterior co-variance, as well as an additional $B^3$ determinant calculations for each pre-computed kernel matrix.

\section{Experiments}
\label{sec:experiments}

We now finish this manuscript with a comprehensive empirical evaluation of our GIBBON acquisition function.  In particular, we consider  batch (Section \ref{subsec:batch}) and multi-fidelity (Section \ref{subsec:MF}) synthetic benchmarks, as-well as well as a molecular design loop over a non-Euclidean and highly-structured search space (Section \ref{subsec:BOSH}). Finally, we examine the performance of GIBBON when inserted into a challenging real-world BO framework that requires both batch and multi-task decision making. Implementations of GIBBON are available in three popular Python libraries for BO: Emukit  \citep{emukit2019}, BoTorch \citep{balandat2020botorch} and Trieste \citep{trieste} .

For clarity, all of our experiments follow a similar format. We run each of the considered BO methods across 50 random seeds, plotting mean performance and a single standard error. For batch algorithms, we count the evaluation of a batch as a single BO iteration. Suboptimality of the current believed optimum $\hat{\textbf{x}}$ is measured by the regret $g(\textbf{x}^*)-g(\hat{\textbf{x}})$, where $\textbf{x}^*$ is the true maximiser. For some experiments we also measure the time taken to choose the next query points (referred to as the optimisation overhead). This computational cost of performing BO includes fitting the GP surrogate model as well as initialising and maximising the acquisition function. All experiments reporting optimisation overheads were performed on a quad core Intel Xeon 2.30GHz processor. 

Across all our experiments, we see the same general behaviour:  GIBBON at least matches, and often exceeds, the performance of existing high-performance acquisition functions whilst incurring an order of magnitudes lower computational overhead. Moreover, the breadth of our experiments showcases that GIBBON is truly a general-purpose acquisition function, forming {a} computationally light-weight acquisition function suitable for standard BO extensions, batch high-cost string design problems and sophisticated synchronous batch multi-task BO frameworks.

Overall, the purpose of our experiments is to demonstrate how GIBBON performs relative to other BO acquisition functions, with a primary focus on existing MES-based approaches. For completeness, we also compare against a range of additional methods, chosen to reflect their popularity, code availability and suitability for the particular experiment. To this end, we compare GIBBON with all the acquisition functions supported by BoTorch and Emukit, as-well as our own implementations of the batch heuristics discussed in Section \ref{sec:batch}. We will introduce these competitors alongside the relevant empirical results. Unfortunately, the PES-based methods discussed in Section \ref{sec:computationalcomplexity} do not have implementations in BoTorch or Emukit. Moreover, we could not find any other comparable maintained software implementations, likely due to demonstrably worse performance of PES than MES \citep[as shown by][]{wang2017max} and PES's difficult-to-implement subroutines (Section \ref{sec:computationalcomplexity}).

\subsection{Standard and Batch Optimisation}
\label{subsec:batch}

For our first set of experiments, we consider a set of synthetic functions provided with the BoTorch package. In particular, we recreate two of the experiments of \cite{balandat2020botorch} by maximising the Hartmann ($d=6$) and Ackley functions ($d=4$), each with observations perturbed by centred Gaussian noise with a variance of $0.25$. In addition, we also consider the Shekel function ($d=4$) under exact observations. For details of these synthetic functions, we refer readers to Appendix \ref{appendix:synth}. Following the setup of \citet{balandat2020botorch}, we initialise all routines by evaluating $2d+2$ random locations, refit our GP's kernel parameters after each BO step, and choose the current believed optimum $\textbf{x}^*$ by maximising the posterior mean of the GP surrogate model. For each experiment, we separately consider purely sequential BO ($B=1$) and batch BO ($B=5$), recording the evaluation of the whole batch as a single optimisation step.
 
{ For all our experiments, we report the performance of GIBBON and Expected Improvement (EI), as well as standard MES (applied to noisy problems by assuming exact observations)}. In addition, we also ran the acquisition functions already supported in BoTorch, i.e Knowledge Gradient (KG), Noisy Expected Improvement (NEI) \citep{picheny2010noisy}, and MFMES (the multi-fidelity MES extension of \cite{takeno2020multi}, used here to support noisy observations). We stress that MFMES was designed to provide computationally light-weight asynchronous batch BO and we will see that its adaptation to synchronous problems (as implemented by BoTorch and discussed earlier in Sections \ref{sec:mes} and \ref{sec:computationalcomplexity}) incurs a substantial computational overhead. For our batch problems, we also implemented BoTorch versions of Local Penalisation (LPEI) and the DPP heuristic (DPPEI) of \cite{kathuria2016batched}, both using EI as their base acquisition function (as considered by \cite{gonzalez2016batch} and \cite{kathuria2016batched}). In addition, we also provide local penalisation with an MES base acquisition function (LPMES), a combination not tested by \cite{gonzalez2016batch} but found to be particularly effective in our experimentation. All MES-based acquisition functions (including GIBBON) use $5$ max-values sampled from a Gumbel distribution fit to surrogate model predictions at $10,000*d$ random locations and are re-sampled for each BO step. All other implementation parameters follow the BoTorch defaults. 

For acquisition function maximisation we use BoTorch's gradient-based maximiser. However, as this inner-loop maximisation can be challenging since it corresponds to a highly multi-modal maximisation across a  $B\times d$-dimensional space. Therefore most batch BO routines build batches greedily by breaking batch design into $B$ separate $d$-dimensional maximisations. Consequently, for all approaches (including our GIBBON acquisition function) except KG , batches are constructed in this greedy manner with a maximisation budget of $10*d$ random restarts for each element of the batch. Although KG is able to jointly allocate batches, its large computational cost restricted us to $20$ restarts (the amount recommended by the BoTorch authors).

\begin{figure}
\centering
\subfloat[Noiseless Shekel ($d=4$, $B=1$)]{\label{noiseless}
    \includegraphics[height=100pt]{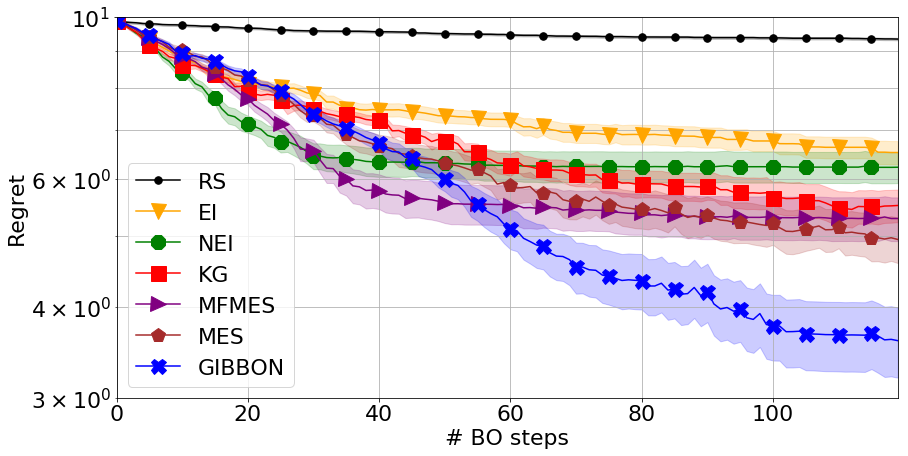}}
\subfloat[Noiseless Shekel ($d=4$, $B=5$)]{
    \includegraphics[height=100pt]{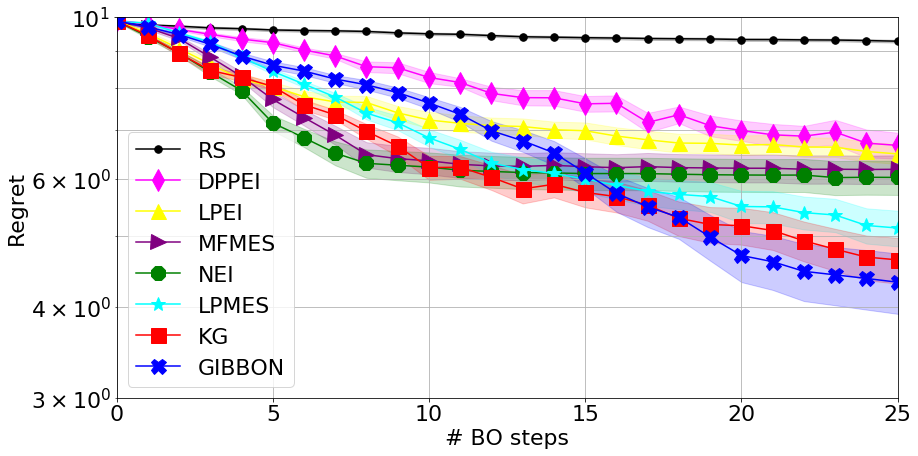}}
\\
\subfloat[Noisy Ackley ($d=4$, $B=1$)]{
    \includegraphics[height=100pt]{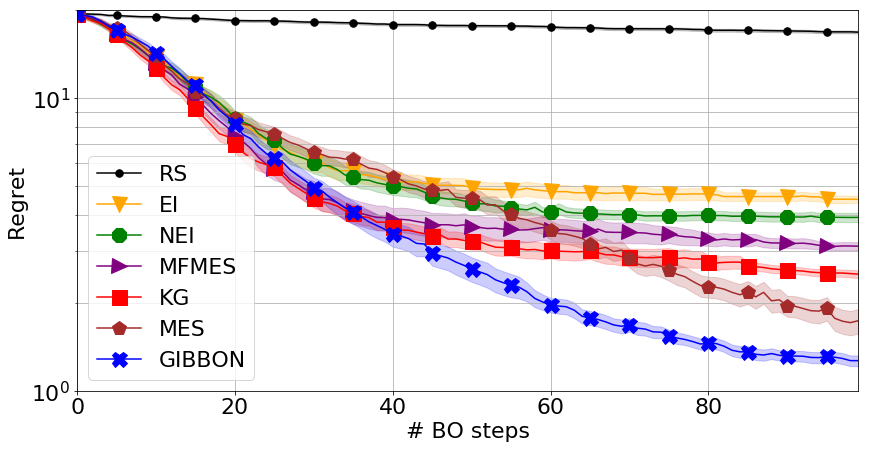}}
\subfloat[Noisy Ackley ($d=4$, $B=5$)]{
    \includegraphics[height=100pt]{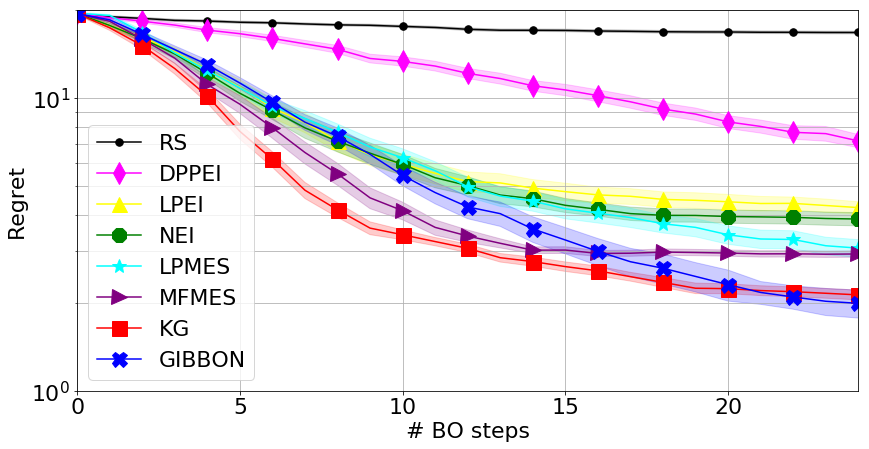}}
\\
\subfloat[Noisy Hartmann ($d=6$, $B=1$)]{
    \includegraphics[height=100pt]{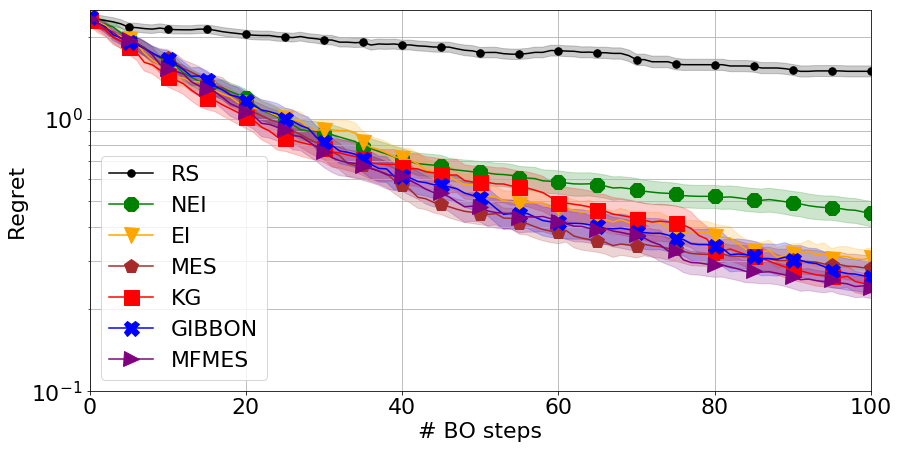}}
\subfloat[Noisy Hartmann ($d=6$, $B=5$)]{
    \includegraphics[height=100pt]{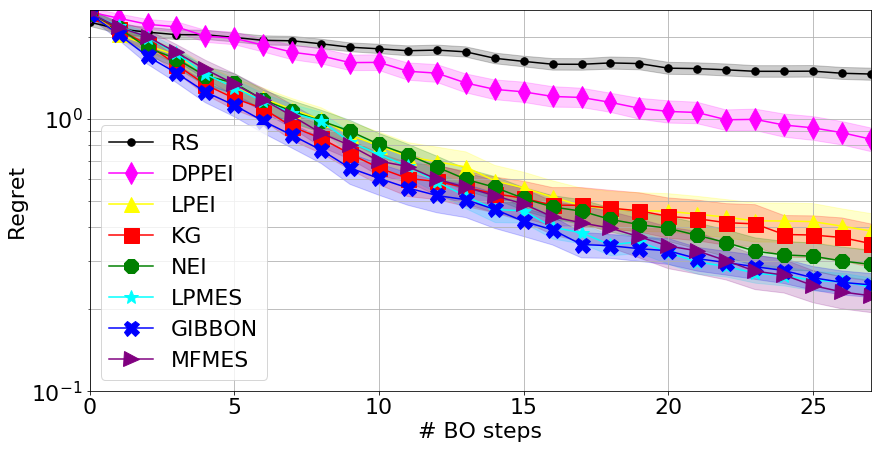}}
\caption{Optimisation of synthetic benchmark functions. GIBBON provides efficient and high-precision optimisation, matching or exceeding the performance of existing approaches.}
\label{fig:synth}
\end{figure}

Across the three synthetic experiments (Figure \ref{fig:synth}) we see that GIBBON provides efficient high-precision optimisation, yielding small regret in competitively few iterations for both sequential and batch BO. {Note that in the noiseless and purely sequential case (Figure \ref{noiseless}), although MFMES can be shown to collapse exactly to the acquisition function of MES, the performance of these two methods differ as BOTorch's implementation of MFMES still relies on numerical approximations (albeit with a a very small observation noise term). Moreover, MFMES's reliance on rough numerical approximations means that its acquisition function struggles to provide high-precision optimisation once the acquisition function values are sufficiently small (i.e towards the end of the optimisation). Consequently, although sometimes achieving fast initial optimisation, MFMES fails to achieve as small final regret as GIBBON.} Surprisingly, GIBBON is able to outperform even standard MES in the noiseless optimisation task of Figure \ref{noiseless}, the scenario for which standard MES is exact. As GIBBON approximates MES, we expected it to perform strictly worse for this example. We delve deeper into this phenomenon in Appendix \ref{appendix:MES}.

Of particular note is the order of magnitude smaller overhead incurred by GIBBON over the other high-performing acquisition functions (NEI, KG and MFMES) as summarised in Table \ref{tab:B=1} (for $B=1$) and Table \ref{tab:B=5} (for $B=5$). In particular, batch KG incurs at least a $10$ times larger overhead than GIBBON. Moreover, Figure \ref{fig:overhead} shows that, while the computational overhead of batch KG, MFMES and NEI increase substantially as the optimisation progresses, { GIBBON's overhead settles to fixed cost. We hypothesise that the initial (small) rise in the computational overhead of GIBBON is caused due to early acquisition functions having wider modes that require more local optimisation steps, a property also likely shared by other acquisition functions but disguised by their growing acquisition function cost.  Although MFMES and GIBBON share the same order complexity with respect to the number of BO steps (see Table \ref{table:complexity}), we see that the large cost of numerical integration renders MFMES significantly more expensive than GIBBON in practice. Moreover, the BoTorch implementation of synchronous batch MFMES employs multiple model fits within each batch allocation to ensure approximation accuracy and so its cost scales poorly with the number of optimisation steps.}

Figure \ref{fig:pareto} confirms our earlier claim that GIBBON is indeed a high-performance yet computationally light-weight acquisition function, showing that GIBBON performs better than all competing acquisition functions while incurring a computational overhead only slightly worse than the simple but low-performance approaches.

\begin{table}
    \centering
\subfloat[Computational overheads for sequential BO ($B=1$).]{
    \begin{tabular}{l|l|l|l}
\hline
 & \multicolumn{3}{l}{Computational Overhead (seconds 1 d.p.)} \\ \cline{2-4} 
 & Shekel (d=4) & Ackley (d=4) & Hartmann (d=6) \\ \hline
EI & 0.2 ($\pm$0.0) & 0.2 ($\pm$0.1) & 0.8 ($\pm$0.1) \\
MES & 0.5 ($\pm$0.1) & \textbf{0.5 ($\pm$0.0)} & 1.0 ($\pm$0.1) \\
NEI & 3.5 ($\pm$0.3) & 3.0 ($\pm$0.2) & 8.9 ($\pm$0.7) \\
MFMES & \textbf{3.0 ($\pm$0.4)} & 0.7 ($\pm$0.1) & 4.5 ($\pm$0.2) \\
KG & 13.0 ($\pm$0.8) & 22 ($\pm$1.0) & \textbf{66.6 ($\pm$4.6)} \\
GIBBON & \textbf{0.6 ($\pm$0.1)} & \textbf{0.8 ($\pm$0.1)} &\textbf{ 1.5 ($\pm$0.1)} \\ \hline
\end{tabular}
    \label{tab:B=1}}
\\
\subfloat[Computational overheads for batch BO ($B=5$)]{  
\begin{tabular}{l|l|l|l}
\hline
 & \multicolumn{3}{l}{Computational Overhead (seconds 1 d.p.)} \\ \cline{2-4} 
 & Shekel (d=4) & Ackley (d=4) & Hartmann (d=6) \\ \hline
DPPEI & 0.8 ($\pm$0.1) & 0.8 ($\pm$0.0) & 1.2 ($\pm$0.0) \\
LPEI & 1.4 ($\pm$0.2) & 2.3 ($\pm$0.1) & 2.9 ($\pm$0.1) \\
LPMES & 2.9 ($\pm$0.1) & 3.3 ($\pm$ 0.1)  & 3.5 ($\pm$ 0.1)  \\
NEI & 21.3 ($\pm$1.8) & 23.4 ($\pm$0.6) & 43.0 ($\pm$2.6) \\
MFMES & 24.4 ($\pm$2.3) & 26.7 ($\pm$0.6) & \textbf{38.6 ($\pm$1.9)} \\
KG & \textbf{58.1 ($\pm$4.4)} & \textbf{53.0 ($\pm$3.1)} & 103.4  ($\pm$6.2) \\
GIBBON & \textbf{5.0 ($\pm$0.5)} & \textbf{5.8 ($\pm$0.7)} & \textbf{13.3 ($\pm$1.3)} \\ \hline
\end{tabular}
    \label{tab:B=5}}
    \caption{Computational overheads for the synthetic benchmarks of Figure \ref{fig:synth} averaged over the whole optimisation run. The two algorithms achieving lowest regret for each task are highlighted, demonstrating that GIBBON at least matches the overhead of other high-performing sequential acquisition functions and incurs a significantly lower overhead than other batch high-performing acquisition functions. }
\end{table}

\begin{figure}
\centering
\subfloat[B=1]{
    \includegraphics[height=100pt]{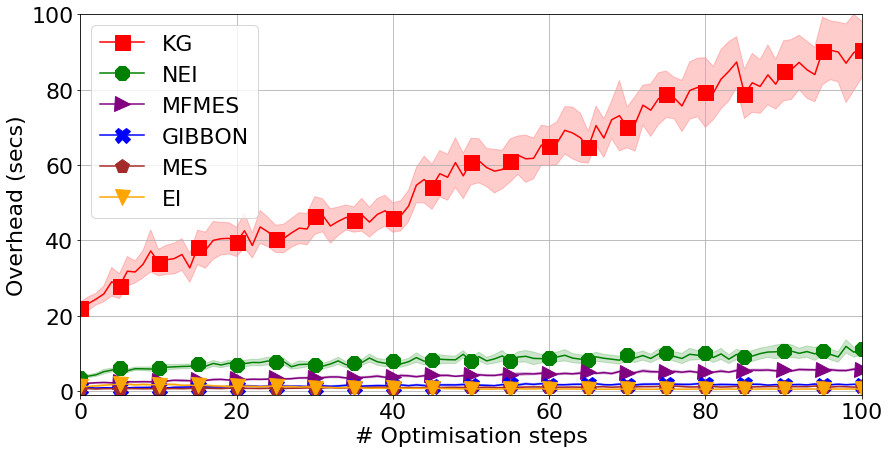}}
\subfloat[B=5]{
    \includegraphics[height=100pt]{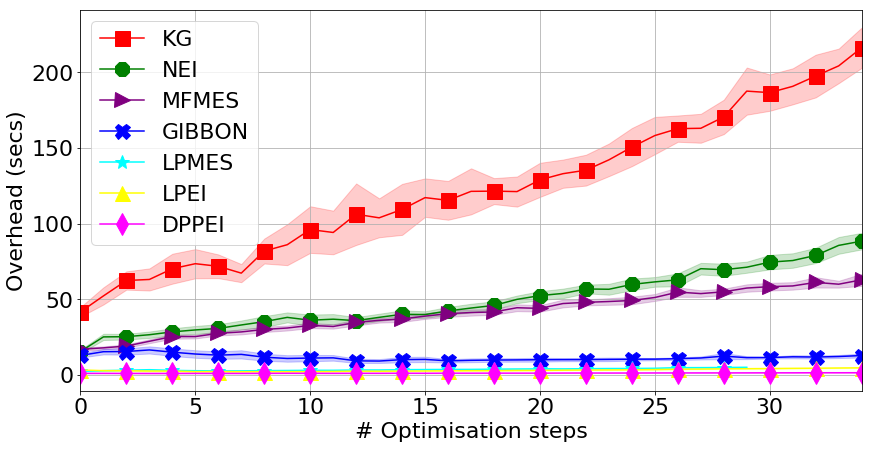}}
\caption{The computational overheads incurred while optimising the Hartmann function. GIBBON's costs remains low throughout the optimisation, whereas the other high-performing batch acquisition functions costs increase dramatically as the optimisation progresses.}
\label{fig:overhead}
\end{figure}

\begin{figure}
\centering
\subfloat[B=1]{
    \includegraphics[width=0.48\textwidth]{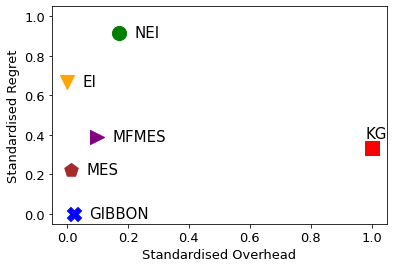}}
\subfloat[B=5]{
    \includegraphics[width=0.48\textwidth]{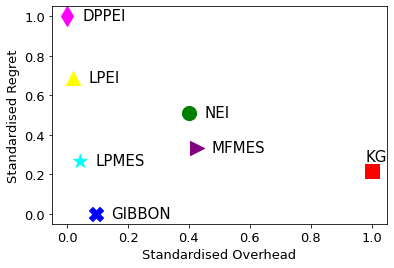}}
\caption{Comparison of the final regret achieved by each BO method with their computational overheads. Scores are standardised to sit within $[0,1]$ and averaged across the three synthetic benchmark tasks.  Lower scores on the x and y axis represent a smaller computational overheads and more effective optimisation, respectively.  }
\label{fig:pareto}
\end{figure}

\subsection{Ablation Study}
\label{subsec:ablation}

Before assessing GIBBON across a wider range of BO tasks, we now perform a brief ablation study into GIBBON's user-controllable parameters and how they affect performance on the noisy Hartmann function introduced above. In particular, we focus on batch size ($B$) and sensitivity to the quality of max-value samples used to calculate GIBBON.

\subsubsection{GIBBON for large batch optimisation}
\begin{figure}
\centering
\subfloat[GIBBON]{
    \includegraphics[height=100pt]{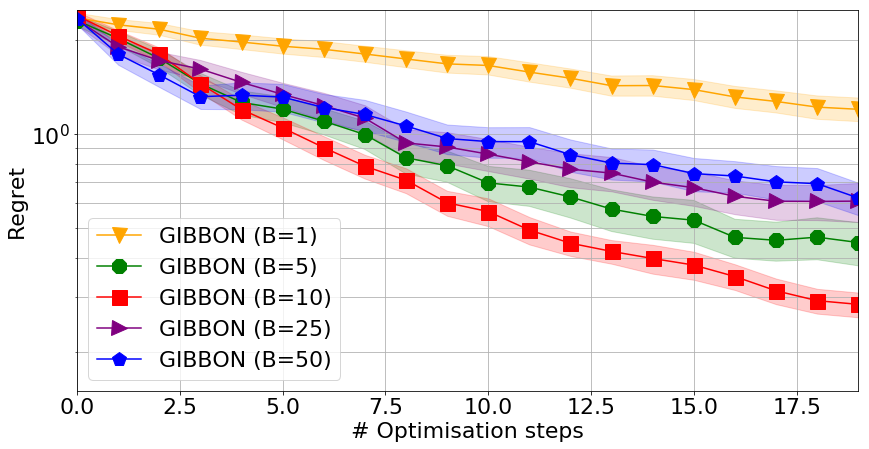}}
\subfloat[LPEI]{
    \includegraphics[height=100pt]{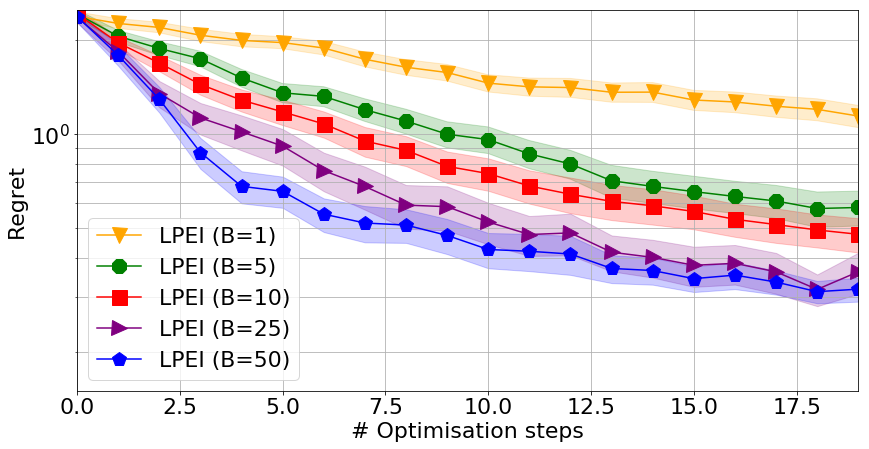}}
\caption{Optimisation of the noisy Hartmann function over 20 iterations across a range of batch sizes. GIBBON is able to provide effective batch optimisation for small to moderate batch size ($B<25$), however, fails to effectively control large batches. Although LPEI is able to leverage larger parallel resources than GIBBON, it still fails to match the performance of GIBBON (B=10) even when controlling much larger batches (B=50).}
\label{fig:batch}
\end{figure}

GIBBON is a promising candidate for optimising under a large degree of parallelism as its batches can be constructed greedily without requiring $B$ posterior updates. Unfortunately, GIBBON fails to realise this promise in practice. Figure \ref{fig:batch} shows that GIBBON fails to effectively leverage large parallel resources and even displays a significant drop in performance once considering batches of size $25$. In contrast, LPEI is able to continually improve regret by considering larger and larger batches. We stress that LPEI, even when controlling batches of $50$ elements (i.e. $1,000$ total evaluations), still achieves lower regret than GIBBON with batches of size $10$ (i.e $200$ evaluations). 

As demonstrated in Appendix \ref{appendix:bigbatch}, GIBBON can be easily modified to support optimisation under large batches by a simple down-weighting of its repulsion term. Therefore, we posit that poor performance of GIBBON in this large batch setting is due to a degradation of the approximation accuracy in our analytical lower bound as we increase batch size. To see this, consider GIBBON's diversity-quality decomposition  first introduced in Section \ref{sec:batch} (i.e. Equation (\ref{decomp})).  Considering large batches ensures that at least some candidate elements must be close together and so have high correlation. Consequently, GIBBON is dominated by its repulsion term (the determinant of the batch's predictive correlation matrix) and the maximisation of GIBBON leads to repeated query points around the edge of the search space, resulting in a substantial degradation in the stability of our GP surrogate model and poor exploration in more important areas of the space. Therefore, in this large batch setting, GIBBON effectively collapses to an almost pure exploration DPP-based method similar to the poorly performing DPPEI examined in our synthetic experiments.

\subsubsection{GIBBON with Thompson-sampled maximum values}

Our proposed calculation strategy for GIBBON requires access to $M$ samples from the objective function's currently unknown maximum value $g^*$. We now investigate the sensitivity of GIBBON with respect to the quality of these random samples. For all the other experiments in this work we used the low-cost but approximate Gumbel sampler, as proposed by \cite{wang2017max}. By approximating the empirical CDF of $g^*$ with an analytical Gumbel distribution, Gumbel sampling is able to return $M$ approximate max-value samples over a grid of $N$ candidate locations with cost $O(M+N)$. Of course, we can access exact samples of $g^*$ by maximising sample functions drawn from our GP (i.e a Thompson-sampling style approach). However, extracting $M$ such exact samples incurs an $O(MN + N^3)$ cost and so, if used as part of GIBBON's calculation strategy, would add significantly to GIBBON's optimisation overhead. However, as using exact max-value samples removes the only  source of approximation in GIBBON aside from our information-theoretic lower bound, this alternative Thompson sampling strategy may lead to improved optimisation --- a hypothesis we now investigate.

\begin{figure}
\centering
\subfloat[Regret (B=1)]{
    \includegraphics[height=100pt]{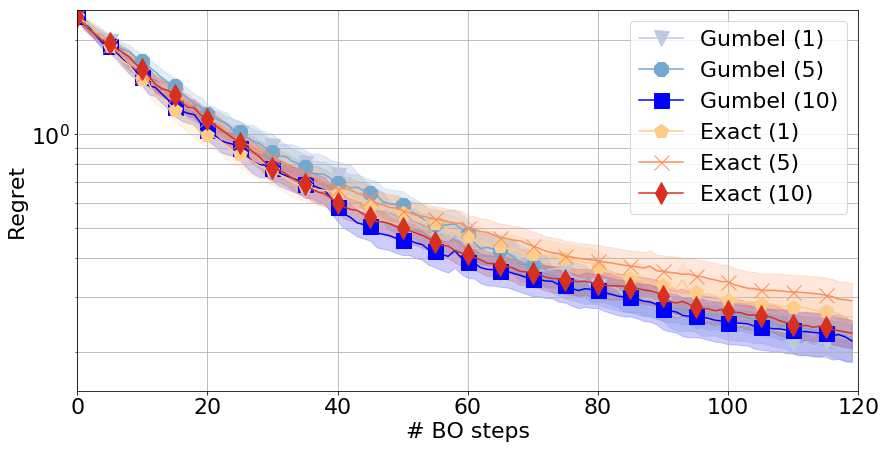}}
\subfloat[Computational Overhead (B=1)]{
    \includegraphics[height=100pt]{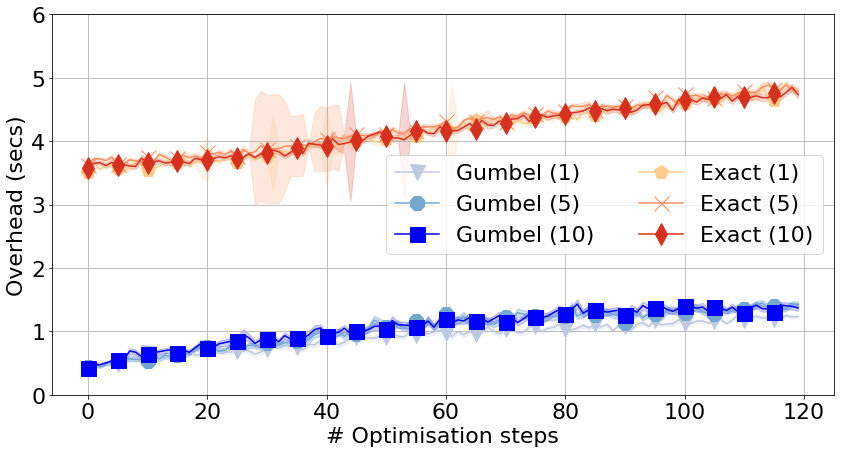}}
    
\subfloat[Regret (B=5)]{
    \includegraphics[height=100pt]{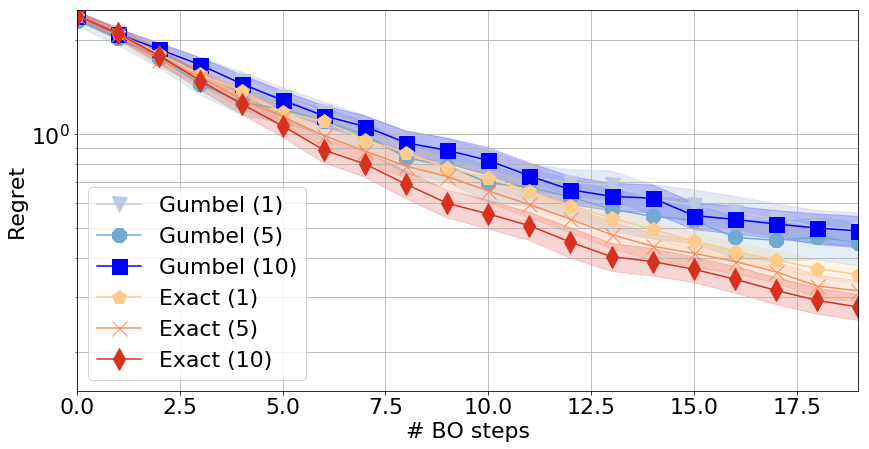}}
\subfloat[Computational Overhead (B=5)]{
    \includegraphics[height=100pt]{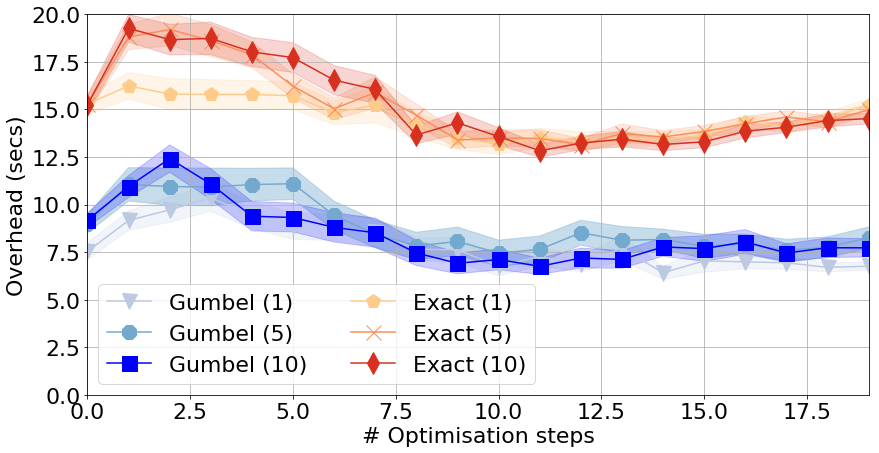}}
\caption{Regret performance (left) and computational overhead (right) of GIBBON when optimising the noisy Hartmann function using batches of size 1 (top row) or 5 (bottom row) across different max-value sampling strategies. Although exact sampling provides a small boost in GIBBON's performance for the batch experiment, Gumbel sampling seems adequate for sequential optimisation. Moreover, the resulting computational overhead of GIBBON with exact sampling is five and three times the cost of GIBBON with Gumbel sampling, when controlling batches of size $B=1$ and $B=5$, respectively.}
\label{fig:samples}
\end{figure}

In Figure \ref{fig:samples}, we present the performance of GIBBON when using $1$, $5$
 and $10$ approximate (Gumbel) or exact (Thompson) sampled maximum values. Due to the significant cost of the exact Thompson sampler, we can sample over only $1,000*d$ random candidate locations, as opposed to the $10,000*d$ used for our Gumbel sampler. We see that, in exchange for a large increase in computational overhead, the exact sampler can sometimes lead to a small increase in performance over our standard Gumbel-based batch GIBBON implementation. We stress that changing sampler had no effect on the performance of purely sequential ($B=1$) BO. This small and inconsistent performance improvement is not enough to justify the additional overhead of exact sampling and so, in order to remain loyal to our motivation of GIBBON as a computationally light acquisition function, we continue using the Gumbel sampler for all our remaining experiments. Investigating alternative sampling strategies to use within GIBBON is an important area of future work.

\color{black}
\subsection{Multi-fidelity Optimisation}
\label{subsec:MF}

We now turn to multi-fidelity optimisation, where the current state-of-the-art acquisition functions are the effectively equivalent MUMBO \citep{mumbo} and MFMES \citep{takeno2020multi} acquisition functions. \cite{mumbo} demonstrates comprehensively that MUMBO outperforms a wide range of existing multi-fidelity acquisition functions, including the entropy search-based approach of \cite{swersky2013multi}, the upper-confidence bound variants of \cite{kandasamy2016gaussian} and \cite{kandasamy2017multi}, as well as extensions of EI \citep{huang2006global} and KG \citep{wu2016parallel}. Therefore, to test GIBBON's multi-fidelity optimisation capabilities, it is sufficient to compare with MUMBO. To this end, we provide an implementation of GIBBON for the Emukit Python library and recreate exactly the synthetic experiments from Figure 2 of \cite{mumbo}. These experiments consider popular synthetic multi-fidelity benchmarks with discrete fidelity spaces consisting of between 2 and 4 fidelity levels (each with differing query costs) and search space dimensions ranging from 2 to 8 dimensions (see Appendix \ref{appendix:synthMF} for the analytical forms of these synthetic benchmarks). In these experiments, we use the linear multi-fidelity GP model of \cite{kennedy2000predicting} as our surrogate model, initialise the GP with a random sample of $2*d$ points queried across all fidelity levels, and fit the GP's kernel parameters to maximise model marginal likelihood after each BO step.

Figure \ref{discretepics} shows that GIBBON provides at least as effective optimisation as MUMBO and Table \ref{MF-overhead} shows that GIBBON has a significantly lighter computational overhead. To provide context for the high performance and low overhead of GIBBON we also present the performance of EI and MES when restricted to just querying the true objective function (i.e no access to low-fidelity observations) and  the performance of the ES acquisition function, used to perform multi-fidelity optimisation by \cite{swersky2013multi}. Although the difference in overhead between MUMBO and GIBBON decreases as we consider higher-dimensional search spaces (primarily due to the growing cost of the Gumbel sampler used by both approaches), the difference in achieved regret increases in GIBBON's favour.

\begin{figure}
\centering
\subfloat[Maximisation of the 2D Currin function (2 fidelity levels with evaluation costs 10 and 1).]{
    \includegraphics[width=0.48\textwidth]{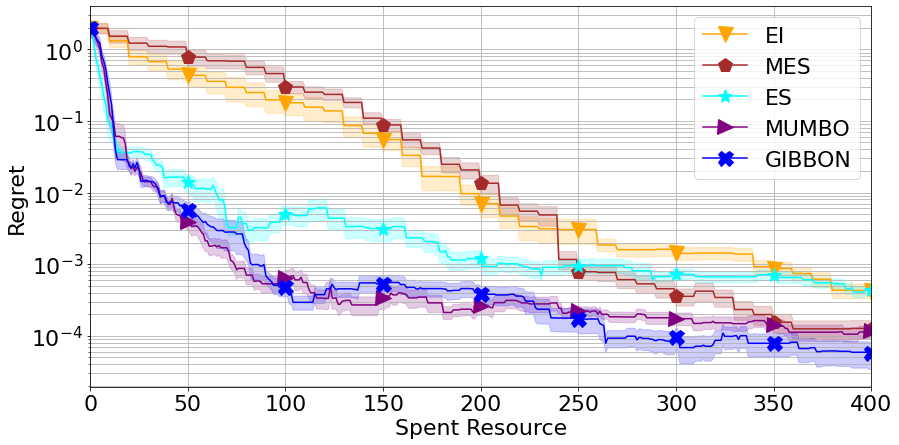}}
\subfloat[Minimisation of 3D Hartmann function (3 fidelity levels with evaluations costs 100, 10 and 1).]{
    \includegraphics[width=0.48\textwidth]{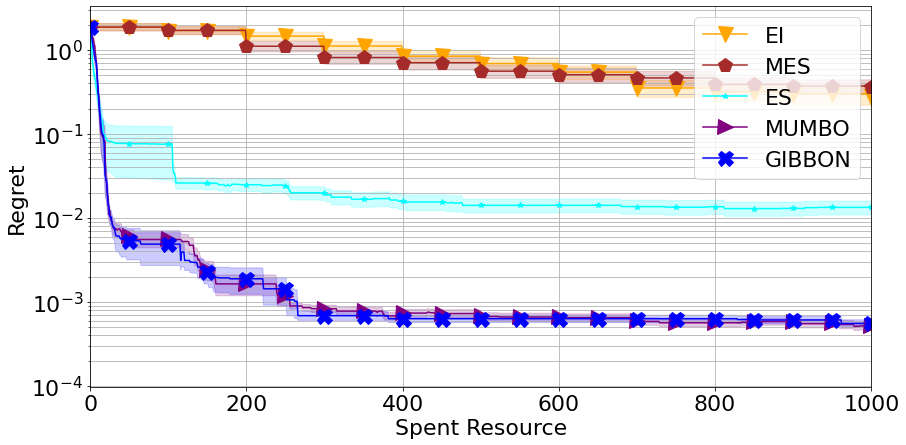}}
\\
\subfloat[Minimisation of 6D Hartmann function (4 fidelity levels).]{
    \includegraphics[width=0.48\textwidth]{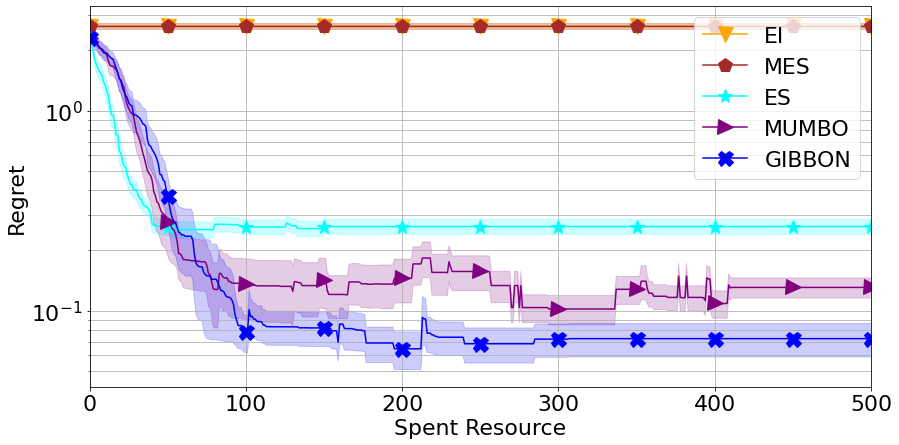}}
\subfloat[Maximisation of the 8D Borehole function (2 fidelity levels with evaluation costs 10 and 1).]{
    \includegraphics[width=0.48\textwidth]{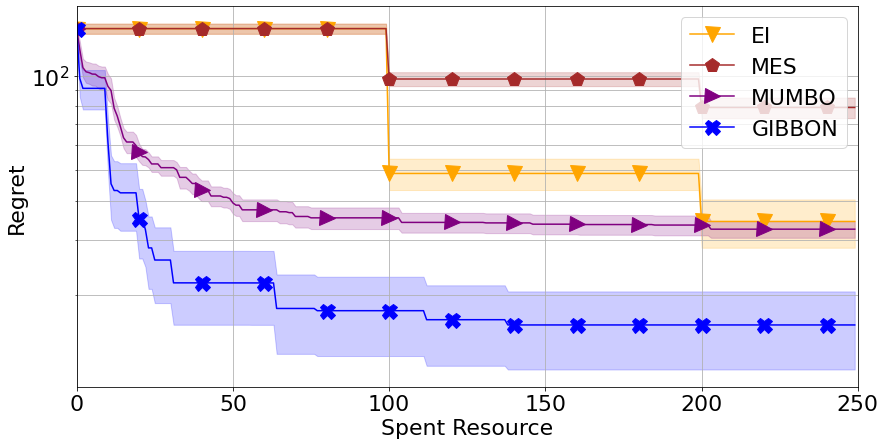}}
 \caption{GIBBON provides high-precision multi-fidelity optimisation with low computational overheads across a range of synthetic multi-fidelity benchmarks. Due to the high-cost of { ES}, we were not able to run it on the higher-dimensional Borehole task. As is standard in multi-fidelity optimisation, the x-axis for these results measures the resources spent on function evaluations (rather than raw BO steps).}
        \label{discretepics}
\end{figure}

\begin{table}
\centering
\begin{tabular}{l|l|l|l|l}
\hline
 & \multicolumn{4}{l}{Overhead for Multi-fidelity Optimisation (Seconds 1 d.p.)} \\ \cline{2-5} 
 & Curin (d=4) & Hartmann (d=3) & Hartmann (d=6) & Borehole (d=8) \\ \hline
ES & 16.6 ($\pm$0.7) & 59.7 ($\pm$4.2) & 229.8 ($\pm$15.3) & - \\
MUMBO & 13.7 ($\pm$0.6) & 18.6 ($\pm$1.0) & 79.9 ($\pm$6.2) & 51.5 ($\pm$7.5)  \\
GIBBON & \textbf{4.0 ($\pm$0.2)} & \textbf{9.9 ($\pm$0.7)} & \textbf{50.2 ($\pm$4.0)} & \textbf{46.1 ($\pm$7.5)} \\ \hline
\end{tabular}
\caption{Computational overheads of the multi-fidelity synthetic benchmarks of Figure \ref{discretepics}. GIBBON enjoys the lowest overheads for all the tasks (as highlighted in bold), often less than half those of MUMBO.}
\label{MF-overhead}
\end{table}

\subsection{Batch Molecular Search}
\label{subsec:string}

BO has recently been applied to high-cost string design problems by \cite{Moss2020}, who consider, among other problems, the task of optimising over molecules. Such tasks are well-suited for BO, due to the high cost of evaluating candidate molecules via wet-lab experiments. \cite{Moss2020} propose a BO framework that fits a GP surrogate model to a popular string-based representation of molecules known as SMILES strings \citep{anderson1987smiles} through a string kernel GP \citep{beck2015learning}. Standard EI arguments are then applied, yielding a highly effective strategy for searching large candidate set of molecules. One practical limitation of this framework, however,  is the large computational cost of string kernels, as incurred for each prediction from the surrogate model GP. Consequently, the framework of \cite{Moss2020} is limited to acquisition functions that require a small number sof surrogate model predictions. Aside from GIBBON, our other considered high-performing batch acquisition functions (MFMES, NEI and KG) require many kernel evaluations for each acquisition function query and the low-cost approaches of DPPEI and LPEI are limited to only Euclidean search spaces. In contrast, GIBBON requires only $B$ surrogate model predictions to measure the utility of a candidate batch and makes no assumptions on the properties of the search space. Therefore, GIBBON can be used to extend the framework of \cite{Moss2020} to batch designs, a property particularly attractive for molecular search applications where it is common practice to synthesis collections of candidate molecules in parallel. 

We now recreate the Zinc example (also considered by \cite{kusner2017grammar} and \cite{griffiths2017constrained}), where we seek to explore a large collection of 250,000 molecules. The task is then to quickly find molecules that score highly according to a chemically-inspired metric, i.e. forming a proxy molecular design loop {with this metric forming our objective function.  As string kernel GPs, which are used to model our moelcules' SMILE strings, have a very large evaluation cost, we cannot evaluate our acquisition function across all the candidate molecules.} Therefore, we  randomly sample $1,000$ molecules for each BO step from which we (greedily) choose to evaluate the $B$ molecules that maximise our GIBBON acquisition function.

We fit our Gumbel sampler on this same sample, re-sampling both the max-values required for GIBBON and the considered $1,000$ molecules at the start of each BO step. We evaluate $20$ randomly chosen molecules to initialise our GP and then allow BO to choose $100$ further molecules, either one by one or as $20$ batches of $5$ molecules or $10$ batches of $10$ molecules. Figure \ref{fig:zinc} shows that even in the purely sequential case, GIBBON provides a modest boost in performance over EI (the acquisition function previously used by \cite{Moss2020}). More importantly, Figure \ref{fig:zinc} also shows that GIBBON is able to provide effective batch optimisation over batches of size 5 and 10, therefore providing an extension of \cite{Moss2020}'s framework where parallel synthesising resources can be used to speed up the molecular search.

\begin{figure}[t]
\centering
\includegraphics[width=0.7\textwidth]{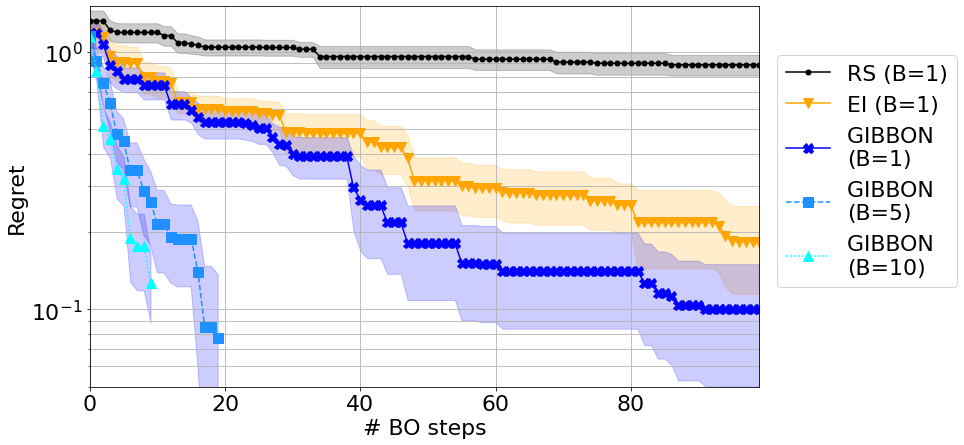}
\caption{ Exploring the Zinc database of molecules with GIBBON. In the purely sequential case, GIBBON finds higher-scoring molecules than EI.  The batched GIBBON approaches reach roughly the same final regret after the same total number of $100$ synthesised molecules { even when GIBBON must choose these evaluations in batches of size 1, 5 or 10. Consequently, GIBBON is able to effectively leverage parallel synthesis resources, reaching the best solutions in fewer BO steps than non-batch alternatives.
For context, we also report the performance of Random Search (RS).}}
\label{fig:zinc}
\end{figure}

\subsection{Bayesian Optimisation by Sampling Hierarchically}
\label{subsec:BOSH}
For our final set of examples, we demonstrate the efficacy of GIBBON as part of a real-world optimisation framework. In particular, we turn to a challenging batch multi-fidelity BO problem inspired by the Knowledge Gradient for Common Random Numbers (KG-CRN) framework of \cite{pearce2019bayesian}. We now provide a brief very overview of the KG-CRN framework and we refer the reader to \cite{pearce2019bayesian} for further details. Our implementation is built upon the Emukit Python package \citep{emukit2019} and was first reported in a workshop paper \citep{moss2020bosh}.

KG-CRN considers BO under highly stochastic evaluations, a scenario where it is commonplace to disregard the original objective function entirely and instead optimise the average of a collection of $K$ specific function realisations, e.g. $K$-fold cross validation (CV) \citep{kohavi1995study} or sample average approximations \citep{kleywegt2002sample}. However, as demonstrated for hyper-parameter tuning \citep{Moss2018}, model selection \citep{moss2019fiesta} and  simulation optimisation \citep{kim2015guide}, optimisation efficiency depends subtly on the choice of $K$. If $K$ is set too low we cannot optimise to high precision, however, setting $K$ too large wastes computation on unnecessarily expensive evaluations. To avoid having to choose $K$ \textit{a-priori}, KG-CRN instead maintains a pool of randomly sampled realisations (e.g. train-test splits or initial environmental conditions) that grows as the optimisation progresses. This construction yields a multi-task BO framework where each individual realisation of the objective function is modelled separately as a perturbation of the true objective function through a Hierarchical Gaussian Process (HGP) \citep{hensman2013hierarchical} (see Appendix \ref{appendix:HGP} for details). Consequently, KG-CRN not only chooses where to evaluate the objective function but also chooses which test problem in which to make the evaluations --- either choosing a member of the previously considered pool of realisations or by generating an entirely new realisation (to be absorbed into the candidate pool for subsequent optimisation steps).

Unfortunately, KG-CRN's acquisition function, a variant of the knowledge gradient of \cite{frazier2008knowledge}, incurs a computational overhead that grows exponentially with the dimensions of the search space and does not support batch optimisation. By replacing this unwieldy acquisition function with GIBBON, we provide our own version of this framework, which we name Bayesian Optimisation Sampled Hierarchically (BOSH). Courtesy of GIBBON, BOSH enjoys small computational overheads and naturally supports batch decision making. We now demonstrate that the GIBBON-based BOSH framework can can provide more efficient and higher-precision optimisation than standard BO across reinforcement learning and hyper-parameter tuning tasks. Full experimental details are provided in Appendix \ref{appendix:experiemnts}. 
 
We report performance across a range of parallel computing resources ($B=1,5,10$), comparing BOSH allocating batches of $B$ points with standard BO routines using the EI and MES acquisition functions to optimise the average of $B$ fixed realisations of the objective function. { For these experiments,  we measure regret as sub-optimality with respect to the best found solution across all methods and replicates.} Unfortunately, \citet{pearce2019bayesian} have yet to provide code for their KG-CRN approach, so we have been unable to provide direct comparisons. However, for the $B=1$ case, we are able to consider FASTCV \citep{swersky2013multi}, an EI-based framework that speeds up optimisation by allowing the evaluation of the individual splits making up $K$-fold CV (for a specific choice of $K$). Unlike our previous examples, we now include the evaluations spent on random initialisation in our plots as the required size of this initialisation is different for each framework (see Appendix \ref{appendix:ini}). { BOSH, for example, is initialised with evaluations at $d+5$ random locations for each of two initial seeds.}

\subsubsection{Reinforcement Learning}
\begin{figure}[t]
\centering
\subfloat[B=1]{\label{LL_1}
    \includegraphics[height=150pt]{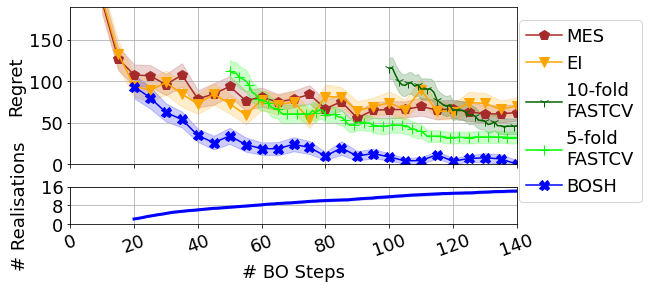}}
\\
\subfloat[B=5]{\label{LL_5}
    \includegraphics[height=130pt]{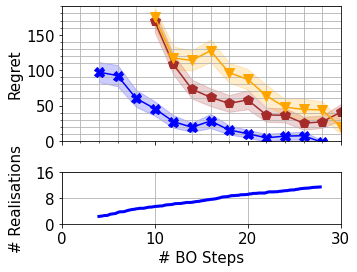}}
\subfloat[B=10]{\label{LL_10}
    \includegraphics[height=130pt]{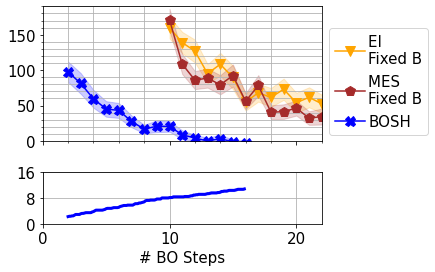}}

\caption{Optimising $7$ parameters of a Lunar Lander controller. We present the regret achieved by each algorithm (top panel) alongside a running count of the number of realisations considered by BOSH (bottom panel). Courtesy of our GIBBON acquisition function, BOSH is able to adaptive consider up to $15$ random conditions to quickly find the optimal controller configuration.}  

\label{LL}
    \end{figure}
For our first experiment, we consider a challenging seven-dimensional stochastic optimisation test-case. We wish to fine-tune a controller for a well-studied reinforcement learning problem, where we must guide a lunar lander across a randomly initialised space to its landing zone with minimal thruster usage (as provided in the OpenAI Gym). Our controller is parameterised by seven unknown constants and a particular configuration can be tested by running a single (or $B$) randomly generated scenarios. We seek to minimise the extra fuel required to land the lander over OpenAI's hard-coded controller (as measured according to a `true' performance measured over a set of  $100$ fixed initial conditions). In this task, each objective function realisation corresponds to an initial environmental condition. As there is substantial variation across different initial conditions, optimising the controller over a small and fixed collection of initialisation fails to provide good `true' performance according to the initial $100$ condition test set (Figure \ref{LL}). Note that FASTCV's need to initialise and then update the large between-realisation correlation matrix severely hampers its optimisation efficiency, as seen by the late start of the corresponding curves in Figure \ref{LL}.

\begin{figure}
\centering
\subfloat[B=1]{\label{IMDB_1}
    \includegraphics[height=150pt]{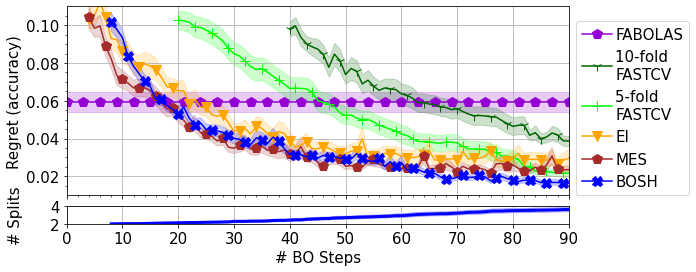}}
\\
\subfloat[B=5]{\label{IMDB_5}
    \includegraphics[height=100pt]{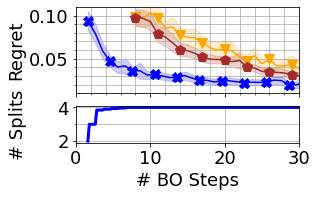}}
\subfloat[B=10]{\label{IMDB_10}
    \includegraphics[height=100pt]{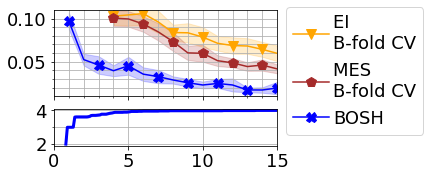}}

\caption{Tuning SVM hyper-parameters for IMDB movie review classification with BOSH. BOSH achieves higher-precision optimisation than all other techniques. When batch computing resources are available, the batch capabilities of  GIBBON allow BOSH to substantially improve optimisation efficiency over standard BO based on $B$-fold CV.}
\label{IMDB_all}
    \end{figure}
\subsubsection{Hyper-parameter Tuning}
We now test the performance of BOSH on a simple ML hyper-parameter tuning task: using a support vector machine (SVM) to classify the sentiment in IMDB movie reviews \citep{maas2011learning}. Here, we seek hyper-parameter values that provide the highest model performance. True model performance is calculated  on a large held-out test set. We stress that these high-cost estimates are only performed retrospectively, after stopping the optimisation, and during the actual tuning our individual performance estimates are generated using a pool of randomly generated train-test splits for BOSH or single train-test splits and $K$-fold CV as fixed evaluation strategies for standard BO.  We also consider the multi-fidelity hyper-parameter tuning framework of FABOLAS \citep{klein2016fast} (following the code provided in \cite{klein2017robo}). As FABOLAS is able to query models using only small proportions of the available data, it is able to find reasonably well performing hyper-parameter configurations in a fraction of the computation used by standard BO and BOSH. However, even if allowed a significantly longer run-time, FABOLAS fails to improve upon this chosen configuration (which we plot as a horizontal line). Figure \ref{IMDB_1} shows that BOSH adaptively considers a pool of up to four train-test splits as the optimisation progresses, providing higher-precision tuning than standard BO based on single train-test splits and substantially faster tuning than standard BO under $5$-fold and $10$-fold cross-validation (Figures \ref{IMDB_5} and \ref{IMDB_10}).

\section{Discussion}

We have presented GIBBON, a general-purpose acquisition function that extends max-value entropy search to provide computationally light-weight yet high performing optimisation for a wide range of BO problems. The efficiency of GIBBON relies on a novel information-theoretical approximation. Moreover, the derivation of this approximation allowed the exploration of the first explicit connection between information-theoretic search, determinantal point process and local penalisation, tying together large sections of the BO literature previously developed and analysed independently.

Not only does GIBBON provide competitive optimisation for common BO extensions like batch and multi-fidelity optimisation, but it forms {} high-performance batch acquisition function suitable for applying BO across highly-structured search spaces, as we demonstrated within a molecular design loop. BO for structured optimisation tasks is a fast growing frontier of the BO literature, with recent work tackling BO for strings \citep{Moss2020,swersky2020amortized}, combinatorial spaces \citep{deshwal2020mercer} and spaces of neural network architectures \citep{kandasamy2018neural}. Therefore, we believe that GIBBON (and our flexible software implementation) will have substantial utility for the machine learning community.

\subsection{Limitations and Future Work}

GIBBON, in its current form, has the two primary practical limitations investigated in our ablation study of Section \ref{subsec:ablation}. Firstly, GIBBON performs poorly for large batch sizes. Improving the large batch capabilities, perhaps though artificially manipulating GIBBON's diversity-quality trade-off, is an important avenue of future work, particularly as the low-cost construction of GIBBON is especially well-suited to large batch scenarios which are currently dominated by simple sampling-based approaches like Thompson sampling \citep{vakili2020scalable}. Secondly, the performance of GIBBON is sensitive to the quality of the max-value samples used within its calculation strategy. Although using a Gumbel sampler to calculate GIBBON provides a truly light-weight acquisition function, we have shown that the  performance of the acquisition function can be improved by considering exact max-value samples. In future work, we will investigate alternative sampling strategies that are more accurate than Gumbel samplers but cheaper than exact Thompson sampling. A promising approach is to follow \cite{hernandez2016predictive} or \cite{takeno2020multi} and employ approximate Thompson sampling methods though kernel decompositions. 

Although shown to be empirically successful, GIBBON has no theoretical guarantees, primarily due to the lack of analysis around our information-theoretical lower bound. Analysing the tightness of this bound could help disentangle which aspects of GIBBON's behaviour are caused by approximation error and which are due to limitations of information-theoretic search strategies in general. In particular, a stronger understanding of approximation quality should explain why GIBBON's performance degrades when building large batches or show exactly when our lower bound is a better approximation of the mutual information than the sampling-based approximations of existing MES extensions. A bound on the approximation error of this bound would also pave the way for convergence guarantees for GIBBON through extensions of the regret bounds of \cite{wang2017max}. To the author's knowledge, no such bound exists for noisy, batch or  multi-fidelity MES-based acquisition functions.

\color{black}

As a final comment, we would like to point out that, although we have already shown GIBBON to have wide applicability, GIBBON can be readily applied to an even wider collection of BO problems. For example, GIBBON can be combined with MESMO \citep{belakaria2019max}, an extension of MES for multi-objective optimisation, to provide {} computationally light-weight acquisition function for batch multi-objective BO. Similarly, GIBBON can also provide a computationally light-weight approach for batch constrained optimisation by extending the MES-based approach of \cite{belakaria2020max}. Finally, GIBBON can be used to improve the performance and reduce the computational cost of any framework relying on batch BO heuristics, for example in non-myopic BO \citep{gonzalez2016glasses,jiang2020binoculars}.

\newpage

\appendix

\section{Extracting The Required Predictive Quantities from a Gaussian Process Surrogate Model}
\label{appendix:quantitites}
We now demonstrate how the distributional quantities required to calculate GIBBON can easily be extracted from a GP surrogate model. For observations $D_n$, let $\textbf{y}_n$ be the already observed evaluations $y$ , and define the kernel matrix $\textbf{K}_n=\left[k(\textbf{z}_i,\textbf{z}_j)\right]_{\textbf{z}_i,\textbf{z}_j\in D_n}$
and kernel vectors $\textbf{k}_n(\textbf{z})=\left[k(\textbf{z}_i,\textbf{z})\right]_{ \textbf{z}_i\in D_n}$ for any valid kernel  defined over the combined search space $\mathcal{Z}=\mathcal{X}\times\mathcal{F}$. Finally, denote the location in the fidelity space corresponding to the true objective function as $\textbf{s}_0$ (i.e $f_{\textbf{s}_0}(\textbf{x})=g(\textbf{x})$). Here, as is standard in multi-fidelity optimisation, we have assumed the ability to query (at least noisily) the true objective function. Then, following \cite{rasmussen2004gaussian} our GP surrogate model provides the following:
\begin{align*}
    \mu^C_i= & \textbf{k}_n((\textbf{x}_i,\textbf{s}_0))^T(\textbf{K}_n+\textrm{diag}(\bm{\sigma}_n))^{-1}\textbf{y}_n \\
    \mu^A_i= &  \textbf{k}_n(\textbf{z}_i)^T(\textbf{K}_n+\textrm{diag}(\bm{\sigma}_n))^{-1}\textbf{y}_n\\
    \Sigma^C_{i,j}= &k((\textbf{x}_i,\textbf{s}_0),(\textbf{x}_j,\textbf{s}_0))-\textbf{k}_n((\textbf{x}_i,\textbf{s}_0))^T(\textbf{K}_n+\textrm{diag}(\bm{\sigma}_n))^{-1}\textbf{k}_n((\textbf{x}_j,\textbf{s}_0))\\
    \Sigma^A_{i,j}= &k(\textbf{z}_i,\textbf{z}_j)-\textbf{k}_n(\textbf{z}_i)^T(\textbf{K}_n+\textrm{diag}(\bm{\sigma}_n))^{-1}\textbf{k}_n(\textbf{z}_j)\\
    \rho_i= &\frac{k(\textbf{z}_i,(\textbf{x}_i,\textbf{s}_0))-\textbf{k}_n(\textbf{z}_i)^T(\textbf{K}_n+\textrm{diag}(\bm{\sigma}_n))^{-1}\textbf{k}_n((\textbf{x}_i,\textbf{s}_0)))}{\sqrt{\Sigma^g_{i,i}\Sigma^y_{i,i}}},
\end{align*} where $\textrm{diag}(\bm{\sigma_n})$ is the $|D_{n}|\times|D_{n}|$ diagonal matrix of observation noises in the evaluations $D_{n}$.

\section{Proof of Theorem~\ref{thm:distribution}}
\label{proofdist}
\dist*
\begin{proof}

As detailed in the main body of this report, we have that \begin{align}
\textbf{C}\sim N(\bm{\mu}^C,\Sigma^C)\quad\textrm{and} \quad A_j \sim  N_1\left(\mu^A_j,\Sigma^A_j\right),\nonumber
\end{align} for some known mean vectors $\bm{\mu}^C,\bm{\mu}^A\in\mathds{R}^B$, a variance vector $\bm{\Sigma}^A\in\mathds{R}^B$ and a co-variance matrix $\Sigma^C\in\mathds{R}^{B\times B}$, as well as a vector $\bm{\rho}\in\mathds{R}^B$ of the correlation between each pair $\{A_j,C_j\}$.  In this section we use $f_{\textbf{X}}$ to denote the probability density function for the random variable $\textbf{X}$ and $f_{\textbf{X},\textbf{Y}}$ to denote the joint probability density function for the random variables $\textbf{X}$ and $\textbf{Y}$.

Now, consider the probability distribution function of random variable of interest:
\begin{align}
    f_{\textbf{A}|C^*\leq m}(\textbf{a})=&\frac{1}{\mathds{P}(C^*\leq m)}\bigints^{\textbf{m}}f_{\textbf{A},\textbf{C}}(\textbf{a},\textbf{b})\:d\textbf{b}\nonumber\\
    =&\frac{1}{\mathds{P}(C^*\leq m)}\bigints^{\textbf{m}}f_{\textbf{A}|\textbf{C}=\textbf{b}}(\textbf{a})f_{\textbf{C}}(\textbf{b})\:d\textbf{b}\nonumber\\=&\frac{1}{\mathds{P}(C^*\leq m)}\bigints^{\textbf{m}}\prod_{i=1}^B\left[f_{A_i|C_i=b_i}(a_i)\right]f_{\textbf{C}}(\textbf{b})\:d\textbf{b},\label{proof1:integral}
\end{align} where $\textbf{b}\in\mathds{R}^B$ and $\textbf{m}=(m,..,m)\in\mathds{R}^B$. The factorisation of $f_{\textbf{A}|\textbf{C}=\textbf{b}}$ is due to the conditional independence of $A_j|C_j$ from $\{C_i\}_{i\neq j}$.

A well-known result for the conditional distribution from a bi-variate Gaussian gives us that for each $i\in\{1,..,B\}$
\begin{align*}
    A_i=a_i|C_i=b_i \sim N_1\left(\mu^A_i+\rho_i\sqrt{\frac{\Sigma^A_{i}}{\Sigma^C_{i,i}}}(b_i-\mu^C_i),(1-\rho^2_i)\Sigma^A_{i}\right),
\end{align*}i.e. we have that
\begin{align}
    \textbf{A}|\textbf{C}=\textbf{b} \sim N\left(\bm{\mu}^A+D(\textbf{b}-\bm{\mu}^C),S\right)\label{helpful},
\end{align}
for diagonal matrices $D,S\in\mathds{R}^B$ with elements $D_{i,i}=\rho_i\sqrt{\frac{\Sigma^A_i}{\Sigma^C_{i,i}}}$ and $S_{i,i}=(1-\rho_i^2)\Sigma^A_i$.

Using (\ref{helpful}), the integrand of (\ref{proof1:integral}) can now be regarded as the product of two b-dimensional Gaussian densities
\begin{align*}
    \left[\prod_{i=1}^bf_{A_i|C_i=b_i}(a_i)\right]f_{\textbf{C}}(\textbf{b}) &= N\left(\textbf{a};\bm{\mu}^A+D(\textbf{b}-\bm{\mu}^C),S\right)*N(\textbf{b};\bm{\mu}^C,\Sigma^C)\\
    &= |D|N\left(\textbf{b};\bm{\mu}^C+D^{-1}(\textbf{a}-\bm{\mu}^A),D^{-1}SD^{-1}\right)*N(\textbf{b};\bm{\mu}^C,\Sigma^C),
\end{align*}
which, using the following standard formula for the product of Gaussians densities
\begin{align*}N(\textbf{x};\textbf{m}_1,\Sigma_1)*N(\textbf{x};\textbf{m}_2,\Sigma_2)=&N(\textbf{m}_1;\textbf{m}_2,\Sigma_1+\Sigma_2)\\& * N(\textbf{x};\left(\Sigma_1^{-1}+\Sigma_2^{-1}\right)^{-1}\left(\sigma_1^{-1}\textbf{m}_1+\Sigma_2^{-1}\textbf{m}_2\right),\left(\Sigma_1^{-1}+\Sigma_2^{-1}\right)^{-1}),\end{align*} can be re-expressed as
\begin{align}
    \left[\prod_{i=1}^bf_{A_i|C_i=b_i}(a_i)\right]f_{\textbf{C}}(\textbf{\textbf{b}}) = |D|
    N&\left(\bm{\mu}^C;\bm{\mu}^C+D^{-1}(\textbf{a}-\bm{\mu}^A),D^{-1}SD^{-1}+\Sigma^C\right)\nonumber\\
    & * N\left(\textbf{b};\bm{\mu}^C+\Sigma^{-1}DS^{-1}(\textbf{a}-\bm{\mu}^A),\Sigma^{-1}\right)\nonumber\\
    =
    N&\left(\textbf{a};\bm{\mu}^A,S + D \Sigma^C D\right)\nonumber\\
    & * N\left(\textbf{b};\bm{\mu}^C+\Sigma^{-1}DS^{-1}(\textbf{a}-\bm{\mu}^A),\Sigma^{-1}\right) \nonumber
\end{align} where $\Sigma=\left(\left(\Sigma^C\right)^{-1}+DS^{-1}D\right)$.

Therefore, we have rewritten the integrand of (\ref{proof1:integral}) as a product of two Gaussian densities, where only one depend on $\textbf{b}$. Consequently, the first Gaussian term can be taken outside the integral, yielding the claimed expression
\begin{align}
    f_{\textbf{A}|C^*<m}(\textbf{a}) = \frac{1}{\mathds{P}(C^*<m)}\phi_{\textbf{X}_1}(\textbf{a})\Phi_{\textbf{X}_2}(\textbf{m}),
    \label{nonumber}
\end{align}
where $\phi_{\textbf{X}_1}$ and $\Phi_{\textbf{X}_2}$ are the probability density and cumulative density functions for the multivariate Gaussian variables
\[\textbf{X}_1\sim \textbf{N}_b\left(\bm{\mu}^A,S +D\Sigma^C D\right)\quad\textrm{and}\quad\textbf{X}_2\sim \textbf{N}_b\left(\bm{\mu}^C+\Sigma^{-1}DS^{-1}(\textbf{a}-\bm{\mu}^A),\Sigma^{-1}\right).\]
\end{proof}

\section{Experimental Details for Synthetic Benchmarks.}
We now provide detailed information about each of our synthetic benchmarks.
\subsection{Standard BO benchmarks}
\label{appendix:synth}

\textbf{Shekel function}. A four-dimensional function with ten local and one global minima defined on $\mathcal{X}\in[0,10]^4$:
\begin{align*}
    f(\textbf{x}) = - \sum_{i=1}^{10}\left(\sum_{j=1}^{4}(x_j-A_{j,i})^2+\beta_i\right)^{-1},
\end{align*}
where 
\begin{align*}
    \beta=\begin{pmatrix}
1\\2\\2\\4\\4\\6\\3\\7\\5\\5
\end{pmatrix}
\quad \textrm{and} \quad
    A=\begin{pmatrix}
4 & 1 & 8 & 6 & 3 & 2 & 5 & 8 & 6 & 7  \\
4 & 1 & 8 & 6 & 7 & 9 & 3 & 1 & 2 & 3.6   \\
4 & 1 & 8 & 6 & 3 & 2 & 5 & 8 & 6 & 7  \\
4 & 1 & 8 & 6 & 7 & 9 & 3 & 1 & 2 & 3.6 
\end{pmatrix}.
\end{align*}

\textbf{Ackley function}. A four-dimensional function with many local minima and a nearly flat outer region surrounding a single global minima defined on $\mathcal{X}\in[-32.768,32.768]^4$:
\begin{align*}
    f(\textbf{x}) = -20 \exp\left(-0.2*\sqrt{\frac{1}{4}\sum_{i=1}^dx_i^2}\right) - \exp\left(\frac{1}{4}\sum_{i=1}^4\cos(2\pi x_i)\right) + 20 +\exp(1).
\end{align*}

\textbf{Hartmann 6 function}. A six-dimensional function with six local minima and a single global minima defined on $\mathcal{X}\in[0,1]^6$:

\begin{align*}
    f(\textbf{x})= -\sum\limits_{i=1}^4\alpha_{i} \exp\left(-\sum\limits_{j=1}^6A_{i,j}(x_j-P_{i,j})^2\right),
\end{align*}

where 
\begin{align*}
    &A=\begin{pmatrix}
10 & 3 & 17 & 3.5 & 1.7 & 8 \\
0.05 & 10 & 17 & 0.1 & 8 & 14 \\
3 & 3.5 & 1.7 & 10 & 17 & 8\\
17 & 8 & 0.05 & 10 & 0.1 & 14
\end{pmatrix},
\quad
&\alpha=\begin{pmatrix}
1 \\
1.2 \\
3 \\
3.2
\end{pmatrix},
\\
\\
    &P=10^{-4}\begin{pmatrix}
1312 & 1696 & 5569 & 124& 8283& 5886\\
2329 & 4135 & 8307 & 3736& 1004& 9991 \\
2348 & 1451 & 3522 & 2883& 3047& 6650 \\
4047 & 8828 & 8732 & 5743& 1091& 381
\end{pmatrix}.
\end{align*}

\subsection{Multi-fidelity benchmarks}
\label{appendix:synthMF}

\textbf{Currin exponential function (discrete fidelity space)}. A two-dimensional function defined on $\mathcal{X}=[0,1]^2$ with two fidelities queried with costs 10 and 1:
\begin{align*}
    f(x_1,x_2,0)=&\left(1-\exp(-\frac{1}{2x_2})\right)\frac{2300x_1^3+1900x_1^2+2092x_1+60}{100x_1^3+500x_1^2+4x_1+20} \\
    f(x_1,x_2,1)=&\frac{1}{4}f(x_1+0.05,x_2+0.05,0)\\+&\frac{1}{4}f(x_1+0.05,max(0,x_2-0.05),0)  \\+&\frac{1}{4}f(x_1-0.05,x_2+0.05,0)\\+&\frac{1}{4}f(x_1-0.05,max(0,x_2-0.05),0).
\end{align*}

\textbf{Hartmann 3 function}. A three-dimensional function with 4 local extrema defined on $\mathcal{X}=[0,1]^3$ with three fidelities ($m=0,1,2$) queried at costs $100,10$ and $1$:

\begin{align*}
    f(x_1,x_2,&x_3,m) = -\sum\limits_{i=1}^4 \alpha_{i,m+1}\exp\left(-\sum\limits_{j=1}^3A_{i,j}(x_j-P_{i,j})^2\right),
\end{align*}
where 
\begin{align*}
    A&=\begin{pmatrix}
3 & 10 & 30 \\
0.1 & 10 & 35 \\
3 & 10 & 30 \\
0.1 & 10 & 35
\end{pmatrix},
\quad
    \alpha=\begin{pmatrix}
1 & 1.01 & 1.02 \\
1.2 & 1.19 & 1.18  \\
3 & 2.9 & 2.8 \\
3.2 & 3.3 & 3.4
\end{pmatrix},
\quad
    P&=\begin{pmatrix}
3689 & 1170 & 2673 \\
4699 & 4387 & 7470 \\
1091 & 8732 & 5547 \\
381 & 5743 & 8828
\end{pmatrix}.
\end{align*}

\textbf{Hartmann 6 function}. A six-dimensional function defined on $\mathcal{X}=[0,1]^6$ with four fidelities ($m=0,1,2,3$) queried at costs $1000,100,10$ and $1$:

\begin{align*}
    f(x_1,x_2,x_3,x_4,x_5,x_6,m)= -\sum\limits_{i=1}^4\alpha_{i,m+1} \exp\left(-\sum\limits_{j=1}^6A_{i,j}(x_j-P_{i,j})^2\right),
\end{align*}

where 
\begin{align*}
    &A=\begin{pmatrix}
10 & 3 & 17 & 3.5 & 1.7 & 8 \\
0.05 & 10 & 17 & 0.1 & 8 & 14 \\
3 & 3.5 & 1.7 & 10 & 17 & 8\\
17 & 8 & 0.05 & 10 & 0.1 & 14
\end{pmatrix},
\quad
&\alpha=\begin{pmatrix}
1 & 1.01 & 1.02 & 1.03\\
1.2 & 1.19 & 1.18 & 1.17\\
3 & 2.9 & 2.8 & 2.7\\
3.2 & 3.3 & 3.4 & 3.5
\end{pmatrix},
\\
\\
    &P=10^{-4}\begin{pmatrix}
1312 & 1696 & 5569 & 124& 8283& 5886\\
2329 & 4135 & 8307 & 3736& 1004& 9991 \\
2348 & 1451 & 3522 & 2883& 3047& 6650 \\
4047 & 8828 & 8732 & 5743& 1091& 381
\end{pmatrix}.
\end{align*}

\textbf{Borehole function}. An eight-dimensional function defined on \begin{align*}
\mathcal{X}=[&0.05,0.15;100,50,000;63070,115600;990,\\&1110;63.1,116;700,820;1120,1680;9855,12055]
\end{align*} with two fidelities queried with costs 10 and 1:
\begin{align*}
    f(\textbf{x},0)=&\frac{2\pi x_3(x_4-x_6)}{\log(x_2/x_1)\left(1+\frac{2x_7x_3}{log(x_2/x_1)x_1^2x_8}+\frac{x_3}{x_5}\right)}, \\
    f(\textbf{x},1)=&\frac{5x_3(x_4-x_6)}{\log(x_2/x_1)\left(1.5+\frac{2x_7x_3}{log(x_2/x_1)x_1^2x_8}+\frac{x_3}{x_5}\right)}. 
\end{align*}

\section{A modified GIBBON for BO with large batches}
\label{appendix:bigbatch}

In Section \ref{subsec:ablation}, we demonstrated that GIBBON fails to effectively control large batches ($B>>10$) and hypothesised that this was due to a dominance of GIBBON's diversity term over its quality term (\ref{decomp}) in these large batch regimes. To support this hypothesis and to propose a variant of GIBBON suitable for large batches, we now investigate a simple modification to GIBBON. In particular, we down-weight GIBBON's diversity term by a factor of $B^2$, with this scaling chosen to reflect the $B^2$ elements present in the the predictive co-variance of a candidate batch of size $B$. Therefore, we have the modified GIBBON acquisition function

\begin{align*}\alpha_n( \{\textbf{z}\}_{i=1}^B) =& \frac{1}{2B^2}\log\big|R\big|+\sum_{i=1}^B\alpha_n^{\textrm{GIBBON}}(\textbf{z}_i) 
\end{align*}
with performance demonstrated in Figure \ref{fig:big_batch}, which repeats the experiment of Section \ref{subsec:ablation}. We see that this simple re scaling is all that is required to allow GIBBON to effectively control large parallel resources.

\begin{figure}
\centering
\subfloat[Standard GIBBON]{
    \includegraphics[height=100pt]{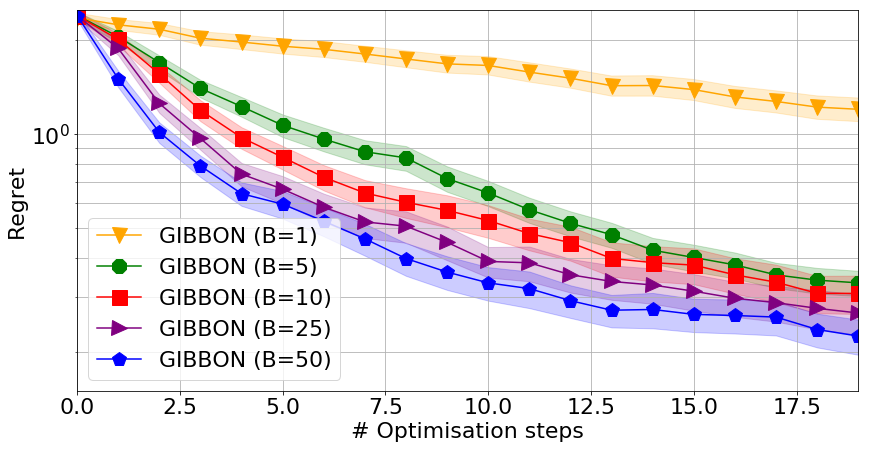}}
\subfloat[Modified]{
    \includegraphics[height=100pt]{GIBBON_batch.png}}
\caption{ Optimisation of the noisy Hartmann function over 20 iterations across a range of batch sizes. Modified GIBBON (left) is able to effectively allocate even the largest batches ($B=50$), achieving faster convergence for each increase in batch size. In contrast, standard GIBBON (right) fails to control batches of size $B>10$.}
\label{fig:big_batch}
\end{figure}

\color{black}
\section{Comparing GIBBON with MES}
\label{appendix:MES}
In our synthetic experiments of Section \ref{sec:experiments}, we were surprised to see that GIBBON was able to outperform MES even in the noiseless standard BO case for which MES provides an exact calculation of entropy reductions. As GIBBON approximates MES, we actually expected GIBBON to perform strictly worse than MES in this particular setting. {
However, we stress that although GIBBON is designed to approximate MES, they are still distinct acquisition functions with differing analytical expressions (see Definition \ref{def:gibbon} and Equation \ref{eq:MES_orig}). Consequently, MES and GIBBON induce (potentially slightly) different exploration-exploitation trade-offs, with the behaviour of GIBBON being particularly well-suited to the Shekel and Ackley, but not Hartmann functions (see Figure \ref{fig:synth}). 

In the specific case (not used in practice) where where we base our MES and GIBBON calculations on a single max-value sample, we can show that GIBBON and MES always choose the same query points (see Section \ref{subsec:equiv} ). However this equivalence does not hold for practical implementations of GIBBON and MES (where we typically use $5$ or $10$ samples of $g^*$)

}

\subsection{Equivalence of the degenerate forms of MES and GIBBON}
\label{subsec:equiv}

To gain further intuition about the relationship between MES and GIBBON, we analyse the so-called degenerate forms their acquisition functions. In the degenerate setting, the acquisition functions are built using only a single max-value sample. By defining the function $u(\textbf{x}) = \frac{m^*-\mu^g_n(\textbf{x})}{\sqrt{\Sigma^g(\textbf{x})}}$, degenerate GIBBON and MES can be expressed as 
\begin{align*}
    \alpha_n^{\textrm{GIBBON}}(\textbf{x})&=-\log\left(1-\frac{\phi(u(\textbf{x}))}{\Phi(u(\textbf{x}))}\left(u(\textbf{x})+\frac{\phi(u(\textbf{x}))}{\Phi(u(\textbf{x}))}\right)\right)\\
    \alpha_n^{\textrm{MES}}(\textbf{x})&= \frac{u(\textbf{x})\phi(u(\textbf{x}))}{2\Phi(u(\textbf{x}))}-\log\Phi(u(\textbf{x})).
\end{align*} Although taking very different analytical forms, these two acquisition functions are strictly decreasing in $u$ (as shown in Figure \ref{fig:acqs}), with GIBBON a strict lower bound on MES. So, in this degenerate and noiseless setting, GIBBON and MES would choose exactly the same points under given exact inner-loop maximisation.

Note that in this degenerate setting,  \cite{wang2017max} provide a bound on the simple regret of degenerate MES. As degenerate GIBBON and degenerate MES choose the same query points, the regret bound of degenerate MES is also inherited by degenerate GIBBON. Although this result does not hold for practical implementations of GIBBON based on multiple samples of $g^*$, or when we perform batch or multi-fidelity BO, the existence of this theoretical guarantee provides reassuring evidence for the validity of our approach.

\begin{figure}[t]
\centering
\includegraphics[width=0.5\textwidth]{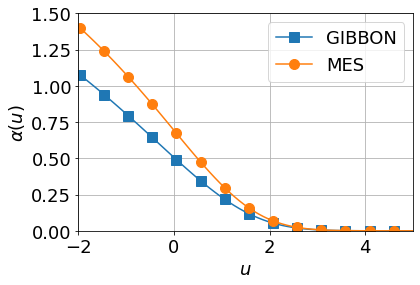}
\caption{Degenerate GIBBON and MES as functions of $u$ {(the standardised difference between the GP posterior at a candidate point and the current estimated maximum value).} The two acquisition functions are monotonically decreasing, taking the same maximiser across a given range of $u$ values. 
\label{fig:acqs}}
\end{figure}

\section{Experimental Details for BOSH }
\label{appendix:experiemnts}
We now provide additional details about our implementation of BOSH and the exact set-ups of our experiments.

\subsection{Hierarchical Gaussian Process}
\label{appendix:HGP}

A natural framework for modelling function realisations as perturbations of a true objective function is a Hierarchical Gaussian Process (HGP) \citep{hensman2013hierarchical}, where the true objective function is modelled as a GP with an `upper' kernel $k_g$, and the deviations to all the individual realisations $f_s$ modelled by another GP with a `lower' kernel $k_f$. As is common in BO, we use Mat\'{e}rn 5/2 kernels \citep{matern1960spatial}. The HGP structure is equivalently understood as  each $f_{s}$ being a conditionally independent GPs with shared mean function $g$, i.e.
\begin{align}
y_i = f_{s_i}(\textbf{x}_i)+\epsilon_i \quad \text{for} \quad
f_{s} \sim \mathcal{GP}(g,k_f) \quad \text{where} \quad
g \sim \mathcal{GP}(0,k_g) 
, \nonumber
\end{align}
for $\epsilon_i\stackrel{\rm i.i.d}{\sim}\mathcal{N}(0,\sigma^2)$. This induces a prior covariance structure of 
\begin{align}
Cov(f_s(\textbf{x}),f_{s'}(\textbf{x}')))=k_g(\textbf{x},\textbf{x}')+\mathds{I}_{s=s'}k_f(\textbf{x},\textbf{x}') \quad \text{and} \quad
Cov(f_s(\textbf{x}),g(\textbf{x}'))=k_g(\textbf{x},\textbf{x}'),\nonumber
\end{align}
where $\mathds{I}$ is an indicator function.

\subsubsection{Predictive Distribution of an HGP}

Crucially, given observations $D_n=\{(\textbf{x}_i,s_i,y_i)\}_{i=1}^n$, the HGP provides a bi-variate Gaussian joint distribution for $(y_s(\textbf{x}),g(\textbf{x}))\,|\,D_n$,  the quantities required to evaluate GIBBON. We will now provide closed form expressions for this joint predictive distributions of $g(\textbf{x})$ and $y_s(\textbf{x})$ given a set of collected evaluations $D_n=\{(\textbf{x}_i,s_i,y_i)\}_{i=1}^n$ (location-realisation-evaluations tuples), where $y_i=f_{s_i}(\textbf{x}_i)+\epsilon$ under Gaussian noise $\epsilon\sim\mathcal{N}(0,\sigma^2)$. 

Defining a compound kernel  $\tilde{k}$ (defined over $\mathcal{X}\times S$) as $\tilde{k}((\textbf{x},s),(\textbf{x}',s'))=k_g(\textbf{x},\textbf{x}')+\mathds{I}_{s=s'}k_f(\textbf{x},\textbf{x}')$ and following \citet{rasmussen2004gaussian} and \citet{hensman2013hierarchical}, our joint posterior distribution can be written as

\begin{eqnarray}
\lefteqn{\begin{pmatrix}g(\textbf{x})\\
y_s(\textbf{x})
\end{pmatrix} \, \bigg|\, D_n}  \quad\quad\quad\quad\quad\quad\sim  N\left[\left(\begin{array}{c}
\mu_n^g(\textbf{x})\\
\mu_n(\textbf{x},s)
\end{array}\right),\left(\begin{array}{ccc}
{\sigma^{g2}_n}(\textbf{x}) & \Sigma_n(\textbf{x},s)\\
\Sigma_n(\textbf{x},s) & \sigma_n^2(\textbf{x},s) + \sigma^2
\end{array}\right)\right],\nonumber
\end{eqnarray}
where 
\begin{align}
\mu_n(\textbf{x},s)=&\tilde{\textbf{k}}_n((\textbf{x},s))^T\left(\tilde{\textbf{K}}_n+\sigma^2I_n\right)^{-1}\textbf{y}_n\nonumber\\
\mu^g_n(\textbf{x})=&\textbf{k}^g_n((\textbf{x},s))^T\left(\tilde{\textbf{K}}_n+\sigma^2I_n\right)^{-1}\textbf{y}_n\nonumber\\
\sigma^2_n(\textbf{x},s)=&\tilde{k}\left((\textbf{x},s),(\textbf{x},s)\right)-\tilde{\textbf{k}}_n((\textbf{x},s))^T\left(\tilde{\textbf{K}}_n+\sigma^2I_n\right)^{-1}\tilde{\textbf{k}}_n((\textbf{x},s))\nonumber\\
\sigma^{g2}_n(\textbf{x})=&k^g\left(\textbf{x},\textbf{x}\right)-\textbf{k}^g_n(\textbf{x})^T\left(\tilde{\textbf{K}}_n+\sigma^2I_n\right)^{-1}\textbf{k}^g_n(\textbf{x})\nonumber\\
\Sigma_n(\textbf{x},s)=&k^g\left((\textbf{x},s),(\textbf{x},s)\right)-\tilde{\textbf{k}}_n((\textbf{x},s))^T\left(\tilde{\textbf{K}}_n+\sigma^2I_n\right)^{-1}\textbf{k}^g_n(\textbf{x}),\nonumber
\end{align}
for $\tilde{\textbf{K}}_n=\left[\tilde{k}((\textbf{x}_i,s_i),(\textbf{x}_j,s_j))\right]_{i,j=1,..,n}$, $\tilde{\textbf{k}}_n((\textbf{x},s))=\left[\tilde{k}((\textbf{x}_i,s_i),(\textbf{x},s))\right]_{ i=1,..,n}$,$\textbf{k}^g_n(\textbf{x})=\left[k_g(\textbf{x}_i,\textbf{x})\right]_{i=1,..n}$ and $\textbf{y}=[y_i]_{i=1,..,n}$. 


Note that predicting from our HGP requires the inversion of the $n\times n$ matrix $\tilde{\textbf{K}}_n+\sigma^2I_n$ and so has comparable cost to predictions from standard GPs.

\subsubsection{BOSH's Kernel Structure}
Our implementation of BOSH uses the following structure for the upper and lower kernels of the HGP:
\begin{align}
    k_g(\textbf{x},\textbf{x}')&=k_{\alpha_g,\beta}(\textbf{x},\textbf{x}')\nonumber\\
    k_f(\textbf{x},\textbf{x}')&=k_{\alpha_f,\beta}(\textbf{x},\textbf{x}')+\sigma^2_f\nonumber,
\end{align}
where $k_{\alpha,\beta}$ denotes the Mat\'{e}rn 5/2 \citep{matern1960spatial} kernel with variance $\alpha\in\mathds{R}$ term and length scales $\beta\in\mathds{R}^d$ hyper-parameters, i.e  
\begin{align}
k_{\alpha,\beta}(\textbf{x},\textbf{x}')=\alpha(1&+\sqrt{5}d_{\beta}(\textbf{x},\textbf{x}')+\frac{5}{3}d_{\beta}(\textbf{x},\textbf{x}')^2)e^{-\sqrt{5}d_{\beta}(\textbf{x},\textbf{x}')},\nonumber
\end{align}
for a weighted distance measure $d_{\beta}(\textbf{x},\textbf{x}')=(\textbf{x}-\textbf{x}')^Tdiag(\beta)(\textbf{x}-\textbf{x}')$. 

As the length-scales are shared between the lower and upper kernels, the total number of kernel parameters for BOSH (including the scale of observation noise $\sigma^2$in our Gaussian likelihood) is $d+4$, only two more than a standard GP with a Mat\'{e}rn 5/2 kernel.

\subsection{Initialisation Costs}
\label{appendix:ini}
Before beginning any BO routine, we must collect an initialisation of points to fit the surrogate model. To allow stable maximisation of the marginal likelihood, it is common to initialise with at least as many evaluations as unknown kernel parameters (to guarantee identifiability). For standard BO, this corresponds to $d+3$ evaluations of the chosen evaluation strategy (i.e requiring $B*(d+3)$ individual function evaluations). For BOSH, rather than using separate lower and upper kernels for our HGP, we found that tying length-scales between each kernel greatly improved the stability of the HGP. Therefore, our HGP has $d+4$ kernel parameters. we allowed BOSH $d+5$ evaluations for each of the seeds in an initial seed pool with two elements. In contrast, reliable initialisation of FASTCV's $B\times B$ correlation matrix (of which its performance was very sensitive) required at least $d+3$ evaluations for each of its $B$ considered seeds. { We found that using fewer initial points severely limited the initial performance of all these methods. } Therefore, as well as providing improved efficiency and precision once optimisation begins, BOSH's ability to model only as many individual seeds as required allows significantly lower initialisation costs.

\subsection{Reinforcement Learning: Lunar Lander}

The Lunar Lander problem is a well-known reinforcement learning task, where we must control three engines (left, main and right) to successfully land a rocket. The learning environment and a hard-coded PID controller is provided in the OpenAI gym \footnote{\textit{https://gym.openai.com/}}. We seek to optimize the $7$ thresholds present in the description of the controller to provide the largest average reward over $100$ random initial conditions. Our RL environment is exactly as provided by OpenAI, with the small modification of randomly initializing the initial lander location (as-well as random initial velocities and terrain) to make a more challenging stochastic optimization problem. We lose 0.3 points per second of fuel use and 100 if we crash. We gain 10 points each time a leg makes contact with the ground, 100 points for any successful landing, and 200 points for a successful landing in the specified landing zone· Each individual run of the environment allows the testing of a controller on a specific random seed.

\subsection{Hyper-parameter Tuning: IMDB SVM}
We tested the performance of BOSH on a real ML problem: tuning a sentiment classification model on the collection of $25,000$ positive and $25,000$ negative IMDB movie reviews used by \cite{maas2011learning}, seeking the hyper-parameter values that provide the model with the highest accuracy. We tune the flexibility of the decision boundary ($\textrm{C}$) and the RBF kernel coefficient ($\textrm{gamma}$) for an SVM \citep{cortes1995support}, a  standard model for binary text classification. As is common in the natural language processing literature, we train our classifier on a bag-of-words representation of the data \citep{jurafsky2014speech}, using tf-idf weightings \citep{salton1988term}. In order to measure the true performance of tuned hyper-parameters, we must use the available data in an unconventional way. By restricting our model fitting and tuning to a randomly sub-sampled $1,000$ review subset to act as our training set for all our experiments, we provide a large held-out collection of $49,000$ movie reviews, upon which we can calculate the `true' performance of the hyper-parameter configurations chosen by our tuning algorithms. We then randomly draw our train-test splits from this fixed training set, with test sets of $10\%$. As already argued, the model scores based on a particular evaluation strategy do not necessarily correspond to the true performance and so, although we acknowledge that this contrived use of the data is not standard, this set-up is necessary to measure the improved efficiency and reliability provided by BOSH.

\bibliography{sample}

\end{document}